\documentclass[11pt]{article}

\def\STYLE{JMLR}

\usepackage{shortcut}
\graphicspath{{../../figures/}}
\def\TikzLocation{../../tikz_figures/}
\def\BibFile{library.bib}

\def\tkzscl{1}
\def\figstyle{_large}

\makeatletter
\def\footnoterule{\kern-3pt \hrule \kern 2.6pt }
\def\ShortHeadings#1#2{\def\ps@jmlrps{\let\@mkboth\@gobbletwo%
\def\@oddhead{\hfill {\small\sc #1} \hfill}%
\def\@oddfoot{\parbox[t]{\textwidth}{%
\hrule
\small\vskip.3em Ongoing work - This document is not finished or considered for publication yet.
			  \center\rm \thepage}}%
\def\@evenhead{\hfill {\small\sc #2} \hfill}%
\def\@evenfoot{\parbox[t]{\textwidth}{%
\hrule
\small\vskip.3em Ongoing work - This document is not finished or considered for publication yet.
			  \center\rm \thepage}}}%
\pagestyle{jmlrps}}
\makeatother

\ShortHeadings{Understanding LISTA}{Moreau and Bruna}

\hypersetup{colorlinks=true, citecolor=blue, linkcolor=purple}

\usepackage{xr}
\def\crossref{\externaldocument{/home/tom/.cache/gummi/.extension.tex}}

\title{\textbf{Understanding the Learned Iterative Soft Thresholding Algorithm with matrix factorization}}
\author{\name Thomas Moreau \email thomas.moreau@cmla.ens-cachan.fr \\
       \addr CMLA, ENS Cachan, CNRS, \\
		Universit\'e Paris-Saclay,\\
		94235 Cachan, France\\
       \AND
       \name Joan Bruna \email bruna@cims.nyu.edu \\
       \addr Courant Institute of Mathematical Sciences,\thanks{Work done while appointed at UC Berkeley, Statistics Department (currently on leave)} \\
	New York University , \\
	New York, NY 10012, USA	 \\}
\date{}
\begin{document}
\def\biblio{}
\def\crossref{}

\editor{}

\maketitle

\begin{abstract}
Sparse coding is a core building block in many data analysis and machine learning pipelines.
Typically it is solved by relying on generic optimization techniques, such as the Iterative Soft Thresholding Algorithm and its accelerated version (ISTA, FISTA).
These methods are optimal in the class of first-order methods for non-smooth, convex functions.
However, they do not exploit the particular structure of the problem at hand nor the input data distribution.
An acceleration using neural networks, coined LISTA, was proposed in \cite{Gregor10}, which showed empirically that one could achieve high quality estimates with few iterations by modifying the parameters of the proximal splitting appropriately.

In this paper we study the reasons for such acceleration.
Our mathematical analysis reveals that it is related to a specific matrix factorization of the Gram kernel of the dictionary, which attempts to nearly diagonalise the kernel with a basis that produces a small perturbation of the $\ell_1$ ball.
When this factorization succeeds, we prove that the resulting splitting algorithm enjoys an improved convergence bound with respect to the non-adaptive version.
Moreover, our analysis also shows that conditions for acceleration occur mostly at the beginning of the iterative process, consistent with numerical experiments.
We further validate our analysis by showing that on dictionaries where this factorization does not exist, adaptive acceleration fails.

\end{abstract}

\crossref{}



\section{Introduction}

Feature selection is a crucial point in high dimensional data analysis.
Different techniques have been developed to tackle this problem efficiently, and amongst them sparsity has emerged as a leading paradigm.
In statistics, the LASSO estimator \citep{Tibshirani1996} provides a reliable way to select features and has been extensively studied in the last two decades (\cite{Hastie2015} and references therein).
In machine learning and signal processing, sparse coding has made its way into several modern architectures, including large scale computer vision \citep{coates2011importance} and biologically inspired models \citep{cadieu2012learning}.
Also, Dictionary learning is a generic unsupervised learning method to perform nonlinear dimensionality reduction with efficient computational complexity \citep{Mairal2009a}.
All these techniques heavily rely on the resolution of $\ell_1$-regularized least squares.

The $\ell_1$-sparse coding problem is defined as solving, for a given input $x \in \R^n$ and dictionary $D \in \R^{n \times m}$, the following problem: 
\begin{equation}
\label{eq:sparsecoding}
z^*(x) = \arg\min_z F_x(z) \overset{\Delta}{=}\frac{1}{2} \| x - Dz \|^2 + \lambda \|z \|_1~.
\end{equation}
This problem  is convex and can therefore be solved using convex optimization machinery.
Proximal splitting methods \citep{Beck2009} alternate between the minimization of the smooth and differentiable part using the gradient information and the minimization of the non-differentiable part using a proximal operator \citep{Combettes2011}.
These methods can also be accelerated by considering a momentum term, as it is done in FISTA \citep{Beck2009, Nesterov2005}.
Coordinate descent \citep{Friedman2007,Osher2009} leverages the closed formula that can be derived for optimizing the problem \autoref{eq:sparsecoding} for one coordinate $z_i$ given that all the other are fixed.
At each step of the algorithm, one coordinate is updated to its optimal value, which yields an inexpensive scheme to perform each step.
The choice of the coordinate to update at each step is critical for the performance of the optimization procedure.
Least Angle Regression (LARS) \citep{Hesterberg2008} is another method that computes the whole LASSO regularization path.
These algorithms all provide an optimization procedure that leverages the local properties of the cost function iteratively.
They can be shown to be optimal among the class of first-order methods for generic convex, non-smooth functions \citep{bubeck2014theory}.

But all these results are given in the worst case and do not use the distribution of the considered problem.
One can thus wonder whether a more efficient algorithm to solve \autoref{eq:sparsecoding} exists for a fixed dictionary $D$ and generic input $x$ drawn from a certain input data distribution.
In \cite{Gregor10}, the authors introduced LISTA, a trained version of ISTA that adapts the parameters of the proximal splitting algorithm to approximate the solution of the LASSO using a finite number of steps.
This method exploits the common structure of the problem to learn a better transform than the generic ISTA step.
As ISTA is composed of a succession of linear operations and piecewise non linearities, the authors use the neural network framework and the backpropagation to derive an efficient procedure solving the LASSO problem.
In \cite{Sprechmann2012}, the authors extended LISTA to more generic sparse coding scenarios and showed that adaptive acceleration is possible under general input distributions and sparsity conditions.

In this paper, we are interested in the following question: Given a finite computational budget, what is the optimum estimator of the sparse coding?
This question belongs to the general topic of computational tradeoffs in statistical inference.
Randomized sketches \citep{alaoui2015fast, yang2015randomized} reduce the size of convex problems by projecting expensive kernel operators into random subspaces, and reveal a tradeoff between computational efficiency and statistical accuracy.
\cite{agarwal2012computational} provides several theoretical results on perfoming inference under various computational constraints, and \cite{chandrasekaran2013computational} considers a hierarchy of convex relaxations that provide practical tradeoffs between accuracy and computational cost.
More recently, \cite{oymak2015sharp} provides sharp time-data tradeoffs in the context of linear inverse problems, showing the existence of a phase transition between the number of measurements and the convergence rate of the resulting recovery optimization algorithm.
\cite{giryes2016tradeoffs} builds on this result to produce an analysis of LISTA that describes acceleration in conditions where the iterative procedure has linear convergence rate.
Finally, \cite{xin2016maximal} also studies the capabilities of Deep Neural networks at approximating sparse inference.
The authors show that unrolled iterations lead to better approximation if one allows the weights to vary at each layer, contrary to standard splitting algorithms.
Whereas their  focus is on relaxing the convergence hypothesis of iterative thresholding algorithms, we study a complementary question, namely when is speedup possible, without assuming strongly convex optimization.
Their results are consistent with ours, since our analysis also shows that learning shared layer weights is less effective.

Inspired by the LISTA architecture, our mathematical analysis reveals that adaptive acceleration is related to a specific matrix factorization of the Gram matrix of the dictionary $B = D\tran D$ as $ B = A\tran S A - R~,$where $A$ is unitary, $S$ is diagonal and the residual is positive semidefinite: $R \succeq 0$.
Our factorization balances between near diagonalization by asking that $\|R \|$ is small and small perturbation of the $\ell_1$ norm, \emph{i.e.} $\| A z \|_1 - \|z \|_1$ is small.
When this factorization succeeds, we prove that the resulting splitting algorithm enjoys a convergence rate with improved constants with respect to the non-adaptive version.
Moreover, our analysis also shows that acceleration is mostly possible at the beginning of the iterative process, when the current estimate is far from the optimal solution, which is consistent with numerical experiments.
We also show that the existence of this factorization is not only sufficient for acceleration, but also necessary.
This is shown by constructing dictionaries whose Gram matrix diagonalizes in a basis that is incoherent with the canonical basis, and verifying that LISTA fails in that case to accelerate with respect to ISTA.

In our numerical experiments, we design a specialized version of LISTA called FacNet, with more constrained parameters, which is then used as a tool to show that our theoretical analysis captures the acceleration mechanism of LISTA.
Our theoretical results can be applied to FacNet and as LISTA is a generalization of this model, it always performs at least as well, showing that the existence of the factorization is a sufficient certificate for acceleration by LISTA.
Reciprocally, we show that for cases where no acceleration is possible with FacNet, the LISTA model also fail to provide acceleration, linking the two speedup mechanisms.
This numerical evidence suggest that the existence of our proposed factorization is sufficient and somewhat necessary for LISTA to show good results.

The rest of the paper is structured as follows.
\autoref{sec:math} presents our mathematical analysis and proves the convergence of the adaptive algorithm as a function of the quality of the matrix factorization.
In \autoref{sec:generic}, we prove that for generic dictionaries, drawn uniformly on the $\ell_2$ unit sphere, it is possible to accelerate ISTA in our framework.
Finally, \autoref{sec:exp} presents the generic architectures that will enable the usage of such schemes and the numerical experiments, which validate our analysis over a range of different scenarios.

\crossref{}

\section{Accelerating Sparse Coding with Sparse Matrix Factorizations}
\label{sec:math}

\subsection{Unitary Proximal Splitting}
In this section we describe our setup for accelerating sparse coding based on the Proximal Splitting method.
Let $\Omega \subset \R^n$ be the set describing our input data, and $D \in \R^{n \times m}$ be a dictionary, with $m > n$.
We wish to find fast and accurate approximations of the sparse coding $z^*(x)$ of any $x \in \Omega$, defined in \autoref{eq:sparsecoding}
For simplicity, we denote $B = D\tran D$ and $y = D^\dagger x$ to rewrite \autoref{eq:sparsecoding} as
\begin{equation}
	\label{sc2}
	z^*(x) = \arg \min_z  F_x(z) = \underbrace{\frac{1}{2}(y -  z)\tran B (y-z)}_{E(z)}  + \underbrace{\vphantom{\frac{1}{2}}\lambda \| z \|_1}_{G(z)} ~.
\end{equation}
For clarity, we will refer to $F_x$ as $F$ and to $z^*(x)$ as $z^*$.
The classic proximal splitting technique finds $z^*$ as the limit of sequence $(z_k)_k$, obtained  by successively constructing a surrogate loss $F_k(z)$ of the form 
\begin{equation}
	F_k(z) = E(z_k) + (z_k-y)\tran B(z-z_k) + L_k \| z - z_k \|_2^2 + \lambda \| z \|_1~,
\end{equation}
satisfying $F_k(z) \geq F(z) $ for all $z \in \R^m$ .
Since $F_k$ is separable in each coordinate of $z$, $z_{k+1} = \arg\min_z F_k(z)$ can be computed efficiently.
This scheme is based on a majoration of the quadratic form $(y-z)\tran B (y-z)$ with an isotropic quadratic form $L_k \| z_k - z \|_2^2$.
The convergence rate of the splitting algorithm is optimized by choosing
$L_k$ as the smallest constant satisfying $F_k(z) \geq F(z)$, which corresponds to the largest singular value of $B$.

The computation of $z_{k+1}$ remains separable by replacing the quadratic form $L_k \textbf{I}$ by any diagonal form.
However, the Gram matrix $B = D\tran D$ might be poorly approximated via diagonal forms for general dictionaries.
Our objective is to accelerate the convergence of this algorithm by finding appropriate factorizations of the matrix $B$ such that 
\[B \approx A\tran S A~,~\text{ and } ~\| A z \|_1 \approx \|z \|_1~,\]
where $A$ is unitary and $S$ is diagonal positive definite.
Given a point $z_k$ at iteration $k$,  we can rewrite $F(z)$ as
\begin{equation}
F(z) = E(z_k)  + ( z_k - y )\tran B(z - z_k) + Q_B(z, z_k)~,
\end{equation}
with $\displaystyle Q_B(v, w ) := \frac{1}{2}(v - w)\tran B (v - w) + \lambda \|v \|_1 ~$.
For any diagonal positive definite matrix $S$ and unitary matrix $A$, the surrogate loss
\mbox{$\widetilde{F}(z,z_k) := E(z_k)  + ( z_k - y)\tran B(z - z_k) + Q_S(A z, Az_k)$}
can be explicitly minimized, since 
\begin{eqnarray}
\label{eq:prox}
\arg\min_z \widetilde{F}(z,z_k) & = & A\tran \arg\min_u \left( (z_k - y )\tran B A\tran (u - Az_k) + Q_S(u, Az_k)\right) \nonumber \\
&=& A\tran \arg\min_u Q_S\left(u, Az_k - S^{-1}AB(z_k - y)\right)
\end{eqnarray}
where we use the variable change $u = Az$.
As $S$ is diagonal positive definite, \autoref{eq:prox} is separable and can be computed easily, using a linear operation followed by a point-wise non linear soft-thresholding.
Thus, any couple $(A, S)$ ensures an computationally cheap scheme.
The question is then how to factorize $B$ using $S$ and $A$ in an optimal manner, that is, such that the resulting proximal splitting sequence converges as fast as possible to the sparse coding solution.

\subsection{Non-asymptotic Analysis}
We will now establish convergence results based on the previous factorization.
These bounds will inform us on 
how to best choose the factors $A_k$ and $S_k$ in each iteration.

For that purpose, let us define 
\begin{equation}
\label{eq:eqdelta}
\delta_A(z) = \lambda \left(\| A z \|_1 - \| z\|_1 \right)~,~\text{and }~R = A\tran S A - B~.
\end{equation}
The quantity $\delta_A(z)$ thus measures how invariant the $\ell_1$ norm is to the unitary operator $A$, whereas $R$ corresponds to the residual of approximating the original Gram matrix $B$ by our factorization $A\tran SA$ .
Given a current estimate $z_k$, we can rewrite
\begin{equation}
\label{eq:surro1}
\widetilde{F}(z,z_k) = F(z) + \frac{1}{2}(z - z_k)\tran R (z - z_k) + \delta_A(z)~.
\end{equation}

By imposing that $R$ is a positive semidefinite residual one immediately obtains the following bound.
\begin{restatable}{proposition}{costUpdate}
	\label{prop:cost_update}
	Suppose that $R = A\tran S A - B$ is positive definite, and define 
	\begin{flalign}
		&& z_{k+1} = \arg\min_z \widetilde{F}(z, z_k)~&.
		\label{eq:algo}\\
		&\makebox[0pt][l]{Then } &F(z_{k+1}) - F(z^*) \leq \frac{1}{2} \| R \| \| z_k - z^* \|_2^2 +&  \delta_A(z^*) - \delta_A(z_{k+1})~.&
		\label{eq:bo1}
	\end{flalign}
\end{restatable}
\begin{proof}
By definition of $z_{k+1}$ and using the fact that $R \succ 0$ we have
\begin{eqnarray*}
F(z_{k+1}) - F(z^*) & \leq & F(z_{k+1}) - \widetilde{F}(z_{k+1},z_k) + \widetilde{F}(z^*,z_k) - F(z^*)  \\
&=& - \frac{1}{2}(z_{k+1} - z_k)\tran R (z_{k+1} - z_k) - \delta_A(z_{k+1}) + \frac{1}{2}(z^* - z_k)\tran R (z^* - z_k) + \delta_A(z^*) \\
&\leq& \frac{1}{2}(z^* - z_k)\tran R (z^* - z_k) + \left(\delta_A(z^*) - \delta_A(z_{k+1})\right) ~.
\end{eqnarray*}
where the first line results from the definition of $z_{k+1}$ and the third line makes use of $R$ positiveness.
\end{proof}

This simple bound reveals that to obtain fast approximations to the sparse coding 
it is sufficient to find $S$ and $A$ such that $\| R \|$ is small and that the $\ell_1$ commutation term $\delta_A$ is small.
These two conditions will be often in tension: one can always obtain $R\equiv 0$ by 
using the Singular Value Decomposition of $B=A_0\tran S_0 A_0$ and setting $A = A_0$ and $S = S_0$.
However, the resulting $A_0$ might introduce large commutation error $\delta_{A_0}$.
Similarly, as the absolute value is non-expansive, \emph{i.e.} $\left| |a| - |b| \right| \leq \left| a - b \right| $,  we have that
\begin{eqnarray}
\label{eq:bo2}
|\delta_A(z)| =  \lambda \left| \| A z \|_1 - \| z \|_1 \right| &\leq & \lambda \| (A - {\bf I}) z \|_1 \\ 
&\leq & \lambda \sqrt{2 \max( \| Az\|_0, \| z \|_0)}~\cdot\, \| A - {\bf I} \| ~\cdot\,\|z \|_2~, \nonumber 
\end{eqnarray}
where we have used the Cauchy-Schwartz inequality $\| x \|_1 \leq \sqrt{ \| x \|_0} \|x \|_2$ in the last equation.
In particular, \autoref{eq:bo2} shows that unitary matrices in the neighborhood of ${\bf I}$ with $\| A - {\bf I} \|$ small have small $\ell_1$ commutation error $\delta_A$ but can be inappropriate to approximate general $B$ matrix.

The commutation error also depends upon the sparsity of $z$ and $Az$~.
If both $z$ and $Az$ are sparse then the commutation error is reduced, which can be achieved if $A$ is itself a sparse unitary matrix.
Moreover, since 
\[|\delta_A(z) - \delta_A(z')| \leq \lambda| \|z \|_1 - \| z' \|_1| + \lambda| \| A z \|_1 - \| A z'\|_1 |~\]
\[ \text{and }~ | \|z \|_1 - \| z'\|_1 | \leq \| z - z' \|_1 \leq \sqrt{\| z- z' \|_0 } \| z - z'\|_2\]
it results that $\delta_A$ is Lipschitz with respect to the Euclidean norm; let us denote by $L_A(z)$ its local Lipschitz constant in z, which can be computed using the norm of the subgradient in $z$\footnote{
This quantity exists as $\delta_A$ is a difference of convex.
See proof of \autoref{prop:bound_error} in appendices for details.
}.
An uniform upper bound for this constant is $(1+\|A\|_1)\lambda\sqrt{m}$, but it is typically much smaller when $z$ and $Az$ are both sparse.\\
Equation \autoref{eq:algo} defines an iterative procedure determined by the pairs $\{(A_k, S_k)\}_k$.
The following theorem uses the previous results to compute an upper bound of the resulting 
sparse coding estimator.
\begin{restatable}{theorem}{convRate}
	\label{thm:conv_rate}
	Let $A_k, S_k$ be the pair of unitary and diagonal matrices corresponding to iteration $k$, 
	chosen such that $R_k = A_k\tran S_k A_k - B \succ 0$.
	It results that
	{\small 
	\begin{equation}
		\label{eq:zo1}
		F(z_{k}) - F(z^*) \leq \frac{(z^*- z_0)\tran R_0 (z^* - z_0) + 2 L_{A_0}(z_1) \|  z^*-z_{1} \|_2 }{2k} + \frac{\alpha - \beta}{2k}~,
	\end{equation}
	\begin{eqnarray*}
	\text{with } & \alpha = &\sum_{i=1}^{k-1} \left(  2L_{A_i}(z_{i+1})\|z^*-z_{i+1} \|_2 + (z^* - z_i)\tran ( R_{i-1} - R_{i}) (z^* - z_i) \right)~,\\
	& \beta = &\sum_{i=0}^{k-1} (i+1)\left((z_{i+1}-z_i)\tran R_i (z_{i+1}-z_i) + 2\delta_{A_i}(z_{i+1}) - 2\delta_{A_i}(z_i) \right)~,
\end{eqnarray*}
}
where $L_A(z)$ denote the local lipschitz constant of $\delta_A$ at $z$.
\end{restatable}

\paragraph{Remark}
If one sets $A_k={\bf I}$ and $S_k = \| B \| {\bf I}$ for all $k \ge 0$,
\autoref{eq:zo1} corresponds to the bound of the ISTA algorithm \citep{Beck2009}.

The proof is deferred to \autoref{sec:proof_conv}.
We can specialize the theorem in the case when $A_0, S_0$ are chosen to minimize the bound \autoref{eq:bo1} and 
$A_k = {\bf I}$, $S_k = \| B \| {\bf I}$ for $k \ge 1$.
\begin{restatable}{corollary}{oneStep}
\label{corr:one_step}
If $A_k = {\bf I}$, $S_k = \| B \| {\bf I}$ for $k \ge 1$ then
{\small 
\begin{equation}
\label{zo22}
F(z_{k}) - F(z^*) \leq \frac{(z^*- z_0)\tran R_0 (z^* - z_0) + 2 L_{A_0}(z_1) (\| z^*-z_1\|  + \| z_1 - z_0\|) + (z^* - z_1)\tran R_0 (z^* - z_1)\tran}{2k}~.
\end{equation}}
\end{restatable}
This corollary shows that by simply replacing the first step of ISTA by the modified proximal step detailed in \autoref{eq:prox}, 
one can obtain an improved bound at fixed $k$ as soon as 
\[2 \| R_0 \| \max(\| z^* - z_0 \|_2^2,\| z^* - z_1 \|_2^2) + 4 L_{A_0}(z_1) \max( \|z^* - z_0\|_2, \|z^* - z_1 \|_2) \leq \|B \| \| z^* - z_0\|_2^2~, \] 
which, assuming $\| z^* - z_0 \|_2 \geq \| z^* - z_1 \|_2 $, translates into
\begin{equation}
\| R_0 \| + 2 \frac{L_{A_0}(z_1)}{\| z^* - z_0\|_2} \leq \frac{\|B \|}{2}~.
\end{equation}
More generally, given a current estimate $z_k$, searching for a factorization $(A_k, S_k)$ will improve 
the upper bound when
\begin{equation}
\label{zo23}
\| R_k \| + 2 \frac{L_{A_k}(z_{k+1})}{\| z^* - z_k\|_2} \leq \frac{\|B \|}{2}~.
\end{equation}
We emphasize that this is not a guarantee of acceleration, since it is based on improving an upper bound.
However, it provides a simple picture on the mechanism that makes non-asymptotic acceleration possible.

\subsection{Interpretation}
\label{sub:interp}

In this section we analyze the consequences of \autoref{thm:conv_rate} in the design of fast sparse coding approximations, and provide a possible explanation for the behavior observed numerically.

\subsubsection{`Phase Transition" and Law of Diminishing Returns}
\label{sub:phase}
\autoref{zo23} reveals that the optimum matrix factorization in terms of minimizing the upper bound depends upon the current scale of the problem, that is, of the distance $\| z^* - z_k\|$.
At the beginning of the optimization, when $\| z^* - z_k\|$ is large, the bound \autoref{zo23} makes it easier to explore the space of factorizations $(A, S)$ with $A$ further away from the identity.
Indeed, the bound tolerates larger increases in $L_{A}(z_{k+1})$, which is dominated by 
\[
L_{A}(z_{k+1}) \leq \lambda (\sqrt{\|z_{k+1}\|_0} + \sqrt{\|Az_{k+1}\|_0})~, 
\]
\emph{i.e.} the sparsity of both $z_1$ and $A_0(z_1)$.
On the other hand, when we reach intermediate solutions $z_k$ such that $\| z^* - z_k\|$ is small with respect to $L_A(z_{k+1})$, the upper bound is minimized by choosing factorizations where $A$ is closer and closer to the identity, leading to the non-adaptive regime of standard ISTA ($A=Id$).

This is consistent with the numerical experiments, which show that the gains provided by learned sparse coding methods are mostly  concentrated in the first iterations.
Once the estimates reach a certain energy level, section \ref{sec:exp} shows that LISTA enters a steady state in which the convergence rate matches that of standard ISTA.

The natural follow-up question is to determine how many layers of adaptive splitting are sufficient before entering the steady regime of convergence.
A conservative estimate of this quantity would require an upper bound of  $\| z^* - z_k\|$ from the energy bound $F(z_k) - F(z^*)$.
Since in general $F$ is convex but not strongly convex, such bound does not exist unless one can assume that $F$ is locally strongly convex (for instance for sufficiently small values of $F$).

\subsubsection{Improving the factorization to particular input distributions}
\label{sub:distrib}
Given an input dataset $\mathcal{D}={(x_i, z^{(0)}_{i}, z^*_i)}_{i \leq N}$, containing examples $x_i \in \R^n$, initial estimates $z^{(0)}_{i}$ and sparse coding solutions $z^*_i$, the factorization adapted to $\mathcal{D}$ is defined as
\begin{equation}
\label{bty}
\min_{A,S;~ A\tran A = {\bf I}, A\tran S A - B \succ 0} \frac{1}{N} \sum_{i \leq N}  \frac{1}{2} ( z^{(0)}_{i} - z^*_i)\tran ( A\tran S A - B ) ( z^{(0)}_{i} - z^*_i) +  \delta_A(z^*_i) - \delta_A(z_{1,i})~.
\end{equation}
Therefore, adapting the factorization to a particular dataset, as opposed to enforcing it uniformly over a given ball $B(z^*; R)$ (where the radius $R$ ensures that the initial value $z_0 \in B(z^*;R)$), will always improve the upper bound \autoref{eq:bo1}.
Studying the gains resulting from the adaptation to the input distribution will be let for future work.

\crossref{}

\section{Generic gap control}
\label{sec:generic}

In this section, we consider the problem of accelerating the resolution of
\autoref{eq:sparsecoding} in the case where $D$ is a generic dictionary,
\emph{i.e.} its elements $D_i$ are draw uniformly over the $\ell_2$ unit-sphere.

\begin{definition}[Generic dictionary]
A dictionary $D \in \R^{p\times K}$ is a generic dictionary when its columns
$D_i$ are drawn uniformly over the $\ell_2$ unit sphere $\mathcal S^{p-1}$.
\end{definition}
The results by \cite{Song1997} show that such dictionaries emerge when the atoms
are drawn independently from normal distributions $\mathcal N(0, \I_p)$ and then
normalized on the unit sphere. Thus, $D_i = \frac{d_i}{\|d_i\|_2}$
with $d_i \sim \mathcal N(0, \I_p)$ for all $i \in \left \{ 1..K \right \}$.
In this context, we consider the matrices $A$ which are perturbation of the
identity and highlight the conditions under which it is possible to find
a perturbation of the identity $A$ which is more advantageous than the identity to
resolve \autoref{eq:sparsecoding}.
For a fixed integer $i \in [K]$, $e_i$ denotes the canonical direction and we introduce $\mathcal E_{\delta,i}$~, the ensemble such that
	\[
		\mathcal E_{\delta,i} = \left \{ u \in \R^K :
			\exists\mu < \delta ,~ 
			\exists h_i\in Span(e_i)^\perp \cap \mathcal S^{K-1}~~
			s.t~~  u = \sqrt{1 - \mu^2}e_i + \mu h_i\right \}~,
	\]
This ensemble contains the vectors which are mainly supported by one of the
canonical directions. Indeed, $\cup_{i=1}^K \mathcal E_{\delta,i} =
\left \{ u\in\R^K : \|u\|_2 = 1, \|u\|_\infty  > \sqrt{1-\delta^2}\right \}$ 
We will denote $A \subset \mathcal E_\delta$ when a matrix $A$ is such that
each of its columns $A_i$ are in $\mathcal E_{\delta, i}$. These matrices
are diagonally dominant and are close to the identity when $\delta$ is close to 0, 
as $\|A - I\|_F = K\delta~.$

%
%
%

\subsection{Control the deviation of the space rotation for $B$}
\label{sub:} 

First, we analyze the possible gain of replacing $B$ by an approximate
diagonalization $A^{-1}SA$ for a diagonally dominant matrix $A\subset \mathcal E_\delta~.$
We choose to study the case where $S$ is chosen deterministicaly when $A$ is fixed.
For $A, B$ fixed, we choose the matrix $S$ which minimizes the frobenius norm of the diagonalization error,
\ie
\begin{equation}
	\label{eq:diagS}
	S = \argmax_{S' diagonal}\|B - A\tran S'A\|_F
\end{equation}
This matrix $S$ can easily be computed as $S_{i,i} = A_i\tran BA_i~.$

\begin{restatable}{lemma}{controlR}
\label{lem:control_Rf}
For a generic dictionary $D$ and a diagonally dominant matrix $A \subset \mathcal E_{\delta}$,   
\begin{align*}
	\E[D]{\min_{A_i \in \mathcal E_{\delta, i}}
		\left\|A^{-1}SA - B\right\|_F^2}
	 \le& \frac{K(K-1)}{p}
	 - 4\delta(K-1)\sqrt{\frac{K}{p}}\\
	 	&~~~+ \delta^2 \left(8\E[D]{\|B\|_F^4} - 6\frac{K(K-1)}{p}\right) 
	 	+ \bO[\delta\to 0]{\delta^3}~. \label{eq:control_R_F}
\end{align*}
\end{restatable}

\begin{proof}{\textbf{sketch for \autoref{lem:control_Rf}.}}
	{\small (The full proof can be found \autoref{sub:control_Rf})}\\
Using the properties of the matrix $A \subset \mathcal E_\delta$ we can show that
\begin{equation}
\label{eq:first_upper}
\left\|A^{-1}SA - B\right\|_F^2 
    \le \left\|B\right\|_F^2(1+8\delta^2K)
	- \sum_{i=1}^K \|DA_i\|_2^4 + \bO[\delta\to 0]{\delta^3}~.
\end{equation}
The first term is the squared Frobenius norm of
a Wishart matrix and we can show
\begin{align*}
\E[D]{\|B\|_F^2}
&= \frac{K(K-1)}{p} + K~.
\end{align*}
The columns $A_i$ are chosen in $\mathcal E_{\delta,i}$, we can thus show that 
\begin{equation}
\label{eq:maxDu}
\E[D]{\max_{u\in\mathcal E_{\delta, i}}\|D u\|_2^4} 
	\ge 1 + 4\delta\E[D]{
			\sqrt{\|D\tran d_i\|_2^2-1}}
		+ 6\delta^2\E[D]{\|D\tran d_i\|_2^2-1}
		+\bO[\delta\to 0]{\delta^3}~.
\end{equation}
	
Denoting $Y_i$ the random variable such that $pY_i^2 = p(\|D\tran d_i\|_2^2-1)$,
we can compute the lower bounds
\[
	\E[D]{Y_i} = \sqrt{\frac{2}{p}}
		\frac{\Gamma\left(\frac{K}{2}\right)}
			 {\Gamma\left(\frac{K-1}{2}\right)}
		\ge \frac{K-1}{\sqrt{pK}}
	~~~~\text{and}~~~~\E[D]{Y_i^2} = \frac{K-1}{p}
\]
Combining these results with \autoref{eq:maxDu} yields the following lower bound when
$\delta \to 0~$,
\begin{align*}
\E[D]{\max_{u\in\mathcal E_{\delta, i}}
	\|D\tran u\|_2^4} 
& \gtrsim 1 + 4\delta\frac{K-1}{\sqrt{pK}} +
		6\delta^2\frac{K-1}{p} + \bO[\delta\to 0]{\delta^3}
\end{align*}
The final bound is obtained using these results with \autoref{eq:first_upper}.
\end{proof}

%
%
%
\subsection{Controling $\E[z \sZ]{\delta_A(z)}$ }
\label{sub:da}

In this subsection, we analyze the deformation of the $\ell_1$-norm due to a rotation
of the code space with a diagonally dominant matrix $A\subset \mathcal E_\delta~.$

\begin{restatable}{lemma}{controlDa}
\label{lem:da}
Let $A\subset \mathcal E_\delta$ be a diagonally dominant matrix and
let $z$ be a {\it random variable} in $\R^K$  with {\it iid} coordinates $z_i$.
Then 
\[
	\E[z, D]{\delta_A(z)}  \le
\left(		\delta \sqrt{K-1} - \frac{\delta^2}{2} +
		\bO[\delta\to 0]{\delta^4}\right)\E[z]{\|z\|_1}
\]
\end{restatable}

\begin{proof}{\textbf{sketch for \autoref{lem:da}.}}
	{\small (The full proof can be found \autoref{sub:control_da})}\\
First, we show that if $z$ is a {\it random variable} in $\R^K$  with {\it iid} coordinates $z_i$, then 
\[
	\E[z,D]{\frac{\|Az\|1}{\|z\|_1} \middle| \|z\|_1} 
	\le \frac{\E[D]{\|A\|_{1,1}}}{K}~.
\]
This permits to decouple the expectations and we obtain the following upper bound
\[
	\E[z]{\delta_A(z)} \le \frac{\E[D]{\|A\|_{1,1}} 
		- \|\I\|_{1,1}}{K} \E[z]{\|z\|_1}~.
\]
Then, for $A \subset \mathcal E_\delta$, the $\ell_1$-norm of the
columns $A_i$ is 
\[
	\E[D]{\|A_i\|_1} \le  \sqrt{1-\delta^2}
		+ \delta \sqrt{K-1}~.
\]
Basic computations permit to show that
\[
	\frac{\E[D]{\|A\|_{1,1}} -I_K}{K}  \le
		\delta \sqrt{K-1} - \frac{\delta^2}{2} +
		\underset{\delta\to 0}{\bO[\delta\to 0]{\delta^4}}~.
\] 
	
\end{proof}

%
%
%

\subsection{Accelerating sparse coding resolution}
\label{sec:certificate}

The two previous results permits to control the upper bound of the cost update defined
in \autoref{prop:cost_update} for generic dictionaries. It si interesting to see when this
upper bound become smaller than the upper bound obtained using the identity $\I$.

\begin{restatable}[Acceleration certificate]{theorem}{certifAcc}
	\label{thm:certif}
	Given a generic dictionary $D$, 
	it is possible find a diagonally dominant matrix $A\subset\mathcal E_\delta$,
	which provide better performance than identity to solve \autoref{eq:sparsecoding} when
	\[
		\lambda\left(\|z\|_1+\|z^*\|\right)
			\le \sqrt{\frac{K(K-1)}{p}}\|z_k-z^*\|_2^2
	\]
\end{restatable}

\begin{proof}{\textbf{sketch for \autoref{thm:certif}.}}
	{\small (The full proof can be found \autoref{sub:certif})}\\
	For $A \subset \mathcal E_\delta$ with columns chosen greedily in
	$\mathcal E_{\delta,i}$, using results from
	\autoref{lem:control_Rf} and \autoref{lem:da},
	\begin{equation}
	\begin{aligned}
	\mathcal E_D \left[\min_{A\subset \mathcal E_\delta}\left\|A^{-1}SA - B\right\|_F^2\right.&\|v\|_2^2
		+ \lambda \delta_A(z)\bigg]
	\le &\\ 
	\frac{(K-1)K}{p} \|v\|_2^2
		&+ \delta\sqrt{K-1} \left(\lambda\|z\|_1 -
				\sqrt{\frac{K(K-1)}{p}}\|v\|_2^2 \right)
		+\bO[\delta\to 0]{\delta^2}
	\end{aligned}	
\end{equation}
	Starting from \autoref{prop:cost_update}, and using the results from 
	 as
	\begin{align*}
		\E[D]{F(z_{k+1}) - F(z^*)} \leq&
		\frac{(K-1)K}{p} \|z_k-z^*\|_2^2\\
		&+ \delta\sqrt{K-1} \underbrace{\left(\lambda\left(\|z\|_1+\|z^*\|\right) -
				\sqrt{\frac{K(K-1)}{p}}\|z_k-z^*\|_2^2 \right)}_{\le 0}
		+\bO[\delta\to 0]{\delta^2}
	\end{align*}
\end{proof}

\begin{figure}[t]
\centering
\includegraphics[width=.7\textwidth]{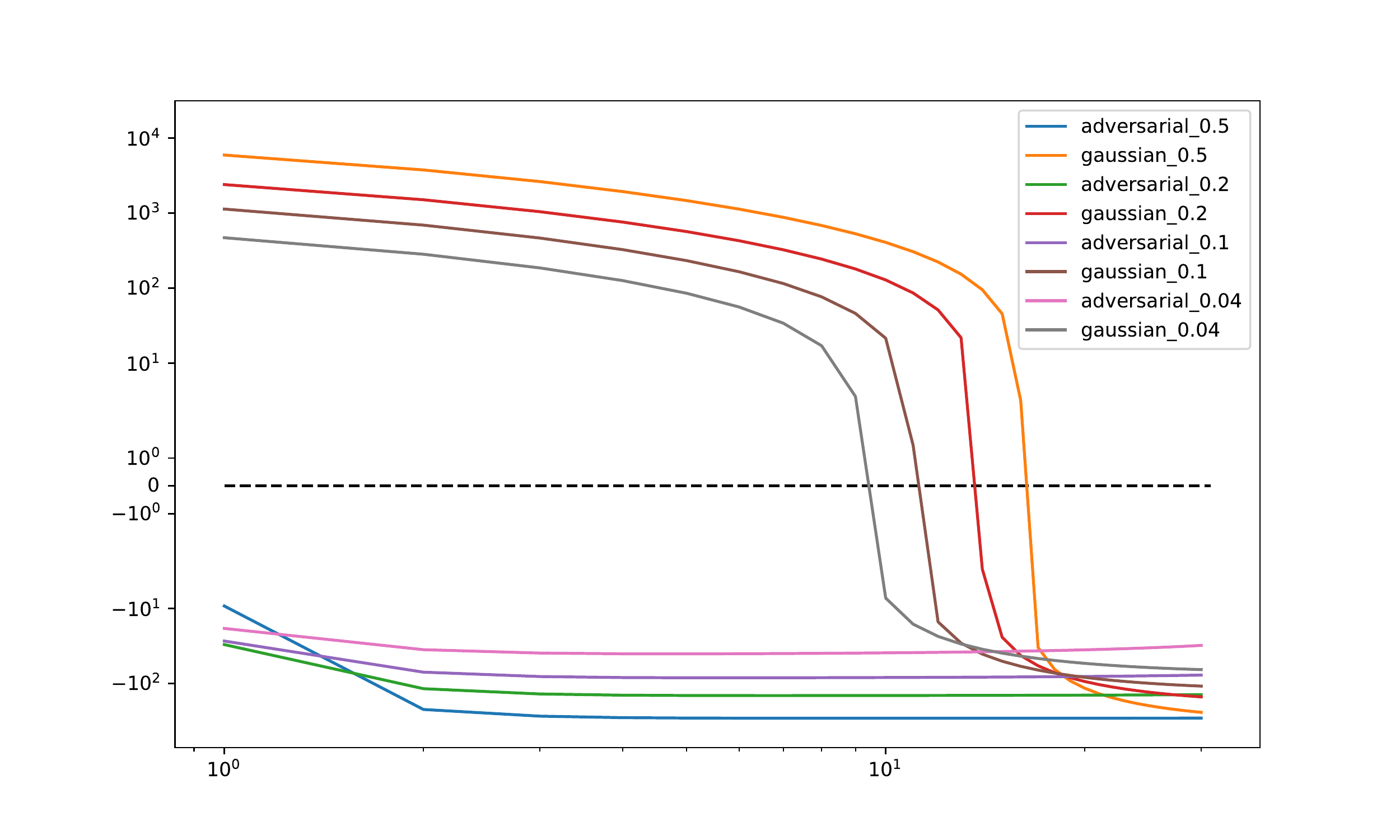}
\caption{Evolution of the gap condition with the
iteration of ISTA on very small problems for an
adversarial dictionary and a gaussian dictionary.}
\label{fig:gap}
\end{figure}

\paragraph{Gap between $\I$ and $A^*$ }
\label{par:gap}

Let $f(A) \overset{\Delta}{=} \frac{1}{2} v\tran
(ABA\tran - S) v + \lambda \delta_A(z)$ be the
function we are trying to minimize. If
\begin{equation}
	\label{eq:cond_gap}
	\lambda\|z\|_1
		\le \sqrt{\frac{K(K-1)}{p}}\|v\|_2^2
\end{equation}
Then for $\delta$ close to $0$, we have $\E[D]{f(A)}
\le \E[D]{f(\I)} = 0$ and there is a gap to improve
the factorization.

\paragraph{Remark 1}
\label{par:remark_I}
This does not depend on $D$ as it is a result in
expectation. But this means that there is a gap for
most of the matrices $D$. 

\paragraph{Remark 2}
\label{par:remark_II}
This gap seems to be consistent with our adversarial
dictionary. See the differences in \autoref{fig:gap}
between a Gaussian dictionary and a worst case,
adversarial one.


\section{Numerical Experiments}
\label{sec:exp}

This section provides numerical arguments to analyse adaptive optimization algorithms and their performances, 
and relates them to the theoretical properties developed in the previous section.
All the experiments were run using Python and Tensorflow.
For all the experiments, the training is performed using Adagrad \citep{duchi2011adaptive}.
The code to reproduce the figures is available online\footnote{The code can be found at ~\url{https://github.com/tomMoral/AdaptiveOptim}}.

\subsection{Adaptive Optimization Networks Architectures}

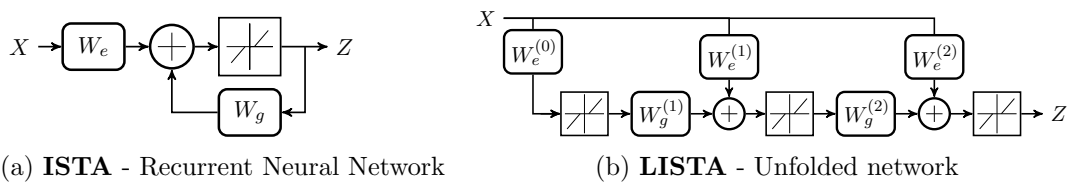
\begin{figure}[t]
\centering
\begin{subfigure}[b]{.4\textwidth}
\scalebox{.8}{

\begin{tikzpicture}[scale=\tkzscl]

\tikzset{
    >=stealth',
    varstyle/.style={
           rectangle,
           fill=white,
           rounded corners,
           draw=black, very thick,
           text width=2em,
           minimum height=2em,
           text centered},
    op/.style={
           circle,
           draw=black, very thick,
           fill=white,
           text width=1em,
           minimum height=1em,
           text centered},
    st/.style={
           draw,
           thick,
           rectangle, 
           fill=white,
           text width=2em, 
           minimum height=2.5em,},
    pil/.style={
           ->,
           thick,
    }
};

\node[varstyle] (linear) {$W_e$};
\node[op, inner sep=4pt, right=1em of linear] (add) {};
\node[left=1em of linear] (X) {$X$} edge[pil] (linear.west);
\path (linear.east) edge[pil] (add.west);


\draw ($ (add.north) - (0, .15)$) -- ($ (add.south) + (0, .15)$);
\draw ($ (add.east) - (.15, 0)$) -- ($ (add.west) + (.15, 0)$);

\node[st, right=1em of add] (h) {};
\node[above=1em of h.center] (up) {};
\node[below=1em of h.center] (down) {};
\node[right=1em of h.center] (right) {};
\node[left=1em of h.center] (left) {};
\draw (h.center) -- (up) -- (down) -- (h.center) -- (right) -- (left);
\draw ($ (h.center) - (.2em,0em)  $) -- ($ (h.center) - (.8em,.7em)  $);
\draw ($ (h.center) + (.2em,0em)  $) -- ($ (h.center) + (.8em,.7em)  $);
\node[right=2em of h] (Z) {$Z$};
\path (add.east)
edge[pil] (h.west);
\path (h.east)
edge[pil] (Z);
\node[inner sep=0,minimum size=0,right=1em of h] (branch) {};
\node[varstyle,below=.5em of h] (S) {$W_g$};
\draw[pil] (branch) |- (S.east);
\draw[pil] (S.west) -| (add);
\end{tikzpicture}
}
\caption{\textbf{ISTA} - Recurrent Neural Network }
\label{fig:ista}
\end{subfigure}
\begin{subfigure}[b]{.55\textwidth}
\scalebox{.75}{

\begin{tikzpicture}[scale=\tkzscl]

\tikzset{
    >=stealth',
    varstyle/.style={
           rectangle,
           fill=white,
           rounded corners,
           draw=black, very thick,
           text width=2em,
           minimum height=2em,
           text centered},
    op/.style={
           circle,
           draw=black, very thick,
           fill=white,
           text width=.5em,
           minimum height=.3em,
           text centered},
    st/.style={
           draw,
           thick,
           rectangle, 
           fill=white,
           text width=1.5em, 
           minimum height=2em,},
    pil/.style={
           ->,
           thick,
    }
};

\def\nlayer{2}

\node (X) {$X$};
\node[right=1em of X] (linear) {}; 
\node[below=3.7em of linear] (p) {};
\node[varstyle, above=1.45em of p] (B) {$W_e^{(0)}$};

\foreach \i in {1,...,\nlayer}{
    \node[st, right=.9em of p] (sta\i) {};
    \path (p.east) edge[pil] (sta\i.west);

    \node[varstyle, right=1em of sta\i] (S) {$W_g^{(\i)}$};
    \node[op, inner sep=4pt, right=1em of S] (p) {};
    \node[varstyle, above=.8em of p] (B\i) {$W_e^{(\i)}$};
    
    \draw ($ (p.north) - (0, .15)$) -- ($ (p.south) + (0, .15)$);
    \draw ($ (p.east) - (.15, 0)$) -- ($ (p.west) + (.15, 0)$);
    
    \draw ($ (sta\i.center) - (0,.9em)  $) -- ($ (sta\i.center) + (0,.9em)  $);
    \draw ($ (sta\i.center) - (.9em,0em)  $) -- ($ (sta\i.center) + (.9em,0em)  $);
    \draw ($ (sta\i.center) - (.2em,0em)  $) -- ($ (sta\i.center) - (.8em,.7em)  $);
    \draw ($ (sta\i.center) + (.2em,0em)  $) -- ($ (sta\i.center) + (.8em,.7em)  $);
    \path (sta\i.east) edge[pil] (S.west);
    \path (S.east) edge[pil] (p.west);
    \draw[pil] (X) -| (B\i) -- (p.north);
}
\draw[pil] (linear.center) -- (B) |- (sta1.west);

\node[st, right=1em of p] (sta) {};
    \draw ($ (sta.center) - (0,.9em)  $) -- ($ (sta.center) + (0,.9em)  $);
    \draw ($ (sta.center) - (.9em,0em)  $) -- ($ (sta.center) + (.9em,0em)  $);
    \draw ($ (sta.center) - (.2em,0em)  $) -- ($ (sta.center) - (.8em,.7em)  $);
    \draw ($ (sta.center) + (.2em,0em)  $) -- ($ (sta.center) + (.8em,.7em)  $);
\node[right=1em of sta] (Z) {$Z$};
\path (p.east) edge[pil] (sta.west);
\path (sta.east) edge[pil] (Z.west);

\end{tikzpicture}
}
\caption{\textbf{LISTA} - Unfolded network}
\label{fig:lista}
\end{subfigure}
\caption{
Network architecture for ISTA/LISTA.
The unfolded version (b) is trainable through backpropagation and permits to approximate the sparse coding solution efficiently.}
\end{figure}

\paragraph{LISTA/LFISTA}
\label{par:lista}
In \cite{Gregor10}, the authors introduced LISTA, a neural network constructed by considering ISTA as a recurrent neural net.
At each step, ISTA performs  the following 2-step procedure :
\begin{equation}
 \left.\begin{aligned}
       1.\hspace{1em} u_{k+1} &= z_k - \displaystyle\frac{1}{L} D\tran(Dz_k-x) = \underbrace{({\bf I} - \frac{1}{L}D\tran D)}_{W_g} z_k + \underbrace{\frac{1}{L}D\tran}_{W_e}x~,\\
       2.\hspace{1em} z_{k+1} &= h_{\frac{\lambda}{L}}(u_{k+1}) \text{ where } h_\theta(u) = \text{sign}(u)(\lvert u\rvert-\theta)_+~,\\
       \end{aligned}
 \right\}
 \quad \text{step $k$ of ISTA}
 \label{eq:ista}
\end{equation}
This procedure combines a linear operation to compute $u_{k+1}$ with an element-wise non linearity.
It can be summarized as a recurrent neural network, presented in \mbox{\autoref{fig:ista}.}, with tied weights.
The autors in \cite{Gregor10} considered the architecture $\Phi_\Theta^K$ with parameters $\Theta = (W_g^{(k)}, W_e^{(k)}, \theta^{(k)})_{k=1,\dots K}$ obtained by unfolding $K$ times the recurrent network, as presented in \autoref{fig:lista}.
The layers $\phi_\Theta^k$ are defined as
\begin{equation}
z_{k+1} = \phi_\Theta^k(z_k) := h_{\theta}( W_g z_k + W_e x)~.
\end{equation}
If $ W_g^{(k)} = {\bf I} - \frac{D\tran D}{L}$, $ W_e^{(k)} = \frac{D\tran}{L}$ and $ \theta^{(k)} = \frac{\lambda}{L}$ are fixed for all the $K$ layers, the output of this neural net is exactly the vector $z_K$ resulting from $K$ steps of ISTA.
With LISTA, the parameters $\Theta$ are learned using back propagation to minimize the cost function: \mbox{$f(\Theta) = \mathbb E_x\left[F_x(\Phi^K_{\Theta}(x))\right]~.$}

A similar algorithm can be derived from FISTA, the accelerated version of ISTA to obtain LFISTA (see \autoref{fig:lfista} in \autoref{sec:lfista}~).
The architecture is very similar to LISTA, now with two memory tapes:
\[z_{k+1} = h_{\theta}( W_g z_k + W_m z_{k-1} + W_e x)~.\]

\paragraph{Factorization network}
\label{par:facnet}

Our analysis in \autoref{sec:math} suggests a refactorization of LISTA in more a structured class of parameters.
Following the same basic architecture, and using \autoref{eq:prox}, the network FacNet, $\Psi_\Theta^K$ is formed using layers such that:
\begin{equation}
	z_{k+1} = \psi_\Theta^k(z_k) := A\tran h_{\lambda S^{-1}} (A z_{k} - S^{-1}A(D\tran D z_{k} - D\tran x) )~,
	\label{eq:facnet}
\end{equation}
with $S$ diagonal and $A$ unitary, the parameters of the $k$-th layer.
The parameters obtained after training such a network with back-propagation can be used with the theory developed in \autoref{sec:math}.
Up to the last linear operation $A\tran$  of the network, this network is a re-parametrization of LISTA in a more constrained parameter space.
Thus, LISTA is a generalization of this proposed network and should have performances at least as good as FacNet, for a fixed number of layers.

The optimization can also be performed using backpropagation.
To enforce the unitary constraints on $A^{(k)}$, the cost function is modified with a penalty: 
\begin{equation}
\label{eq:facnet}
f(\Theta) = \mathbb E_x\left[F_x(\Psi^K_{\Theta}(x))\right] + \frac{\mu}{K}\sum_{k=1}^K\left\|{\bf I} - \left(A^{(k)}\right)^ TA^{(k)}\right\|^2_2~,
\end{equation}
with $\Theta = (A^{(k)}, S^{(k)})_{k=1\dots K}$ the parameters of the K layers and $\mu$ a scaling factor for the regularization.
The resulting matrix $A^{(k)}$ is then projected on the Stiefel Manifold using a SVD to obtain final parameters, coherent with the network structure.

\paragraph{Linear model}
\label{par:linear}

Finally, it is important to distinguish the performance gain resulting from choosing a suitable starting point and the acceleration from our model.
To highlights the gain obtain by changing the starting point, we considered a linear model with one layer such that $z_{out} = A^{(0)} x$.
This model is learned using SGD with the convex cost function \mbox{$ f(A^{(0)}) = \|({\bf I} - DA^{(0)})x\|_2^2+\lambda \|A^{(0)}x\|_1~$}.
It computes a tradeoff between starting from the sparsest point $\bf 0$ and a point with minimal reconstruction error $y~$.
Then, we observe the performance of the classical iteration of ISTA using $z_{out}$ as a stating point instead of $\bf 0~$.


\subsection{Synthetic problems with known distributions}
\label{sub:gen}
\begin{figure}[t]
\centering
    \begin{subfigure}[t]{0.48\textwidth}
		\includegraphics[width=\textwidth]{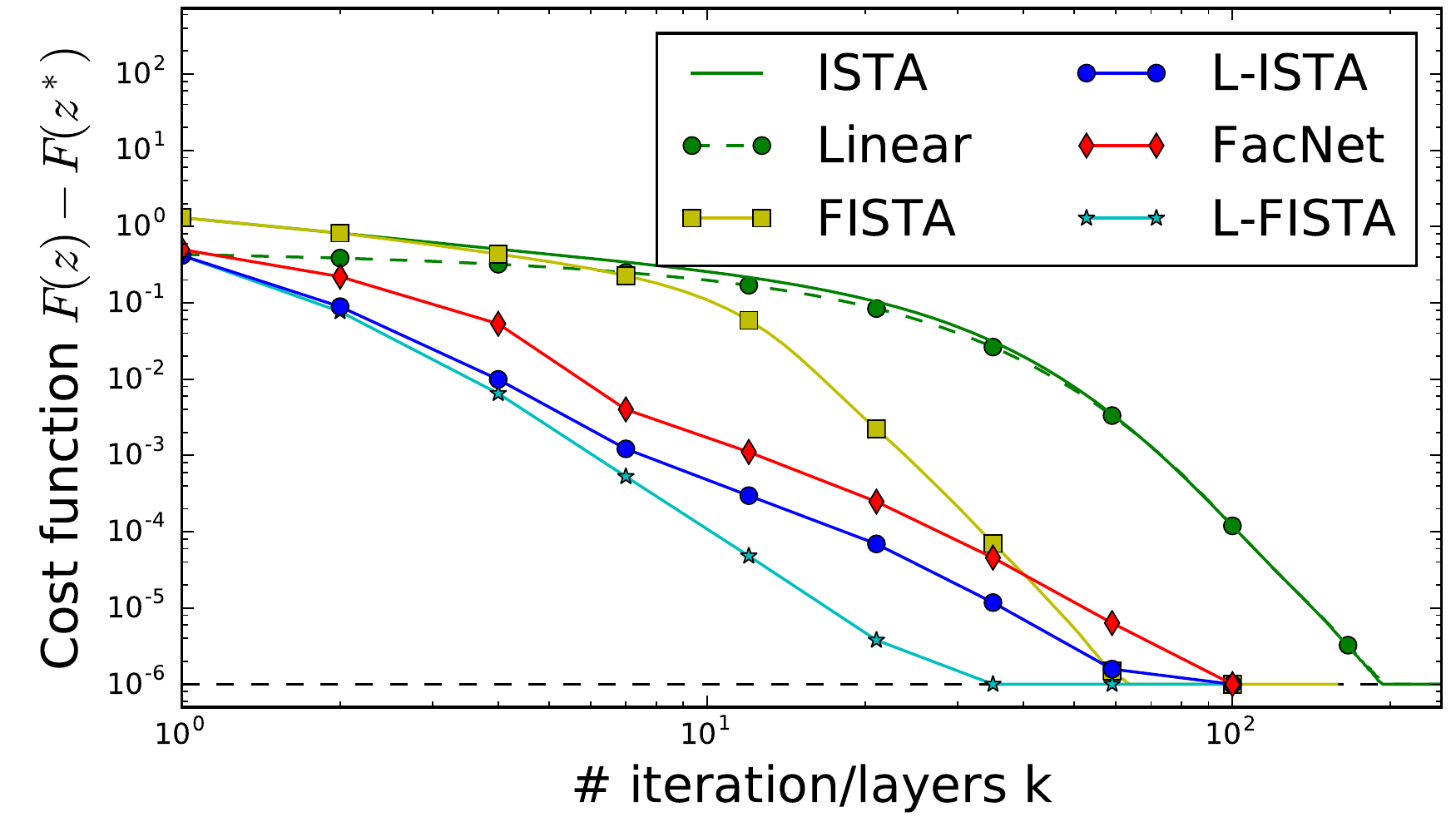}
		\label{fig:layers1}
    \end{subfigure}
    \begin{subfigure}[t]{0.48\textwidth}
        \includegraphics[width=\textwidth]{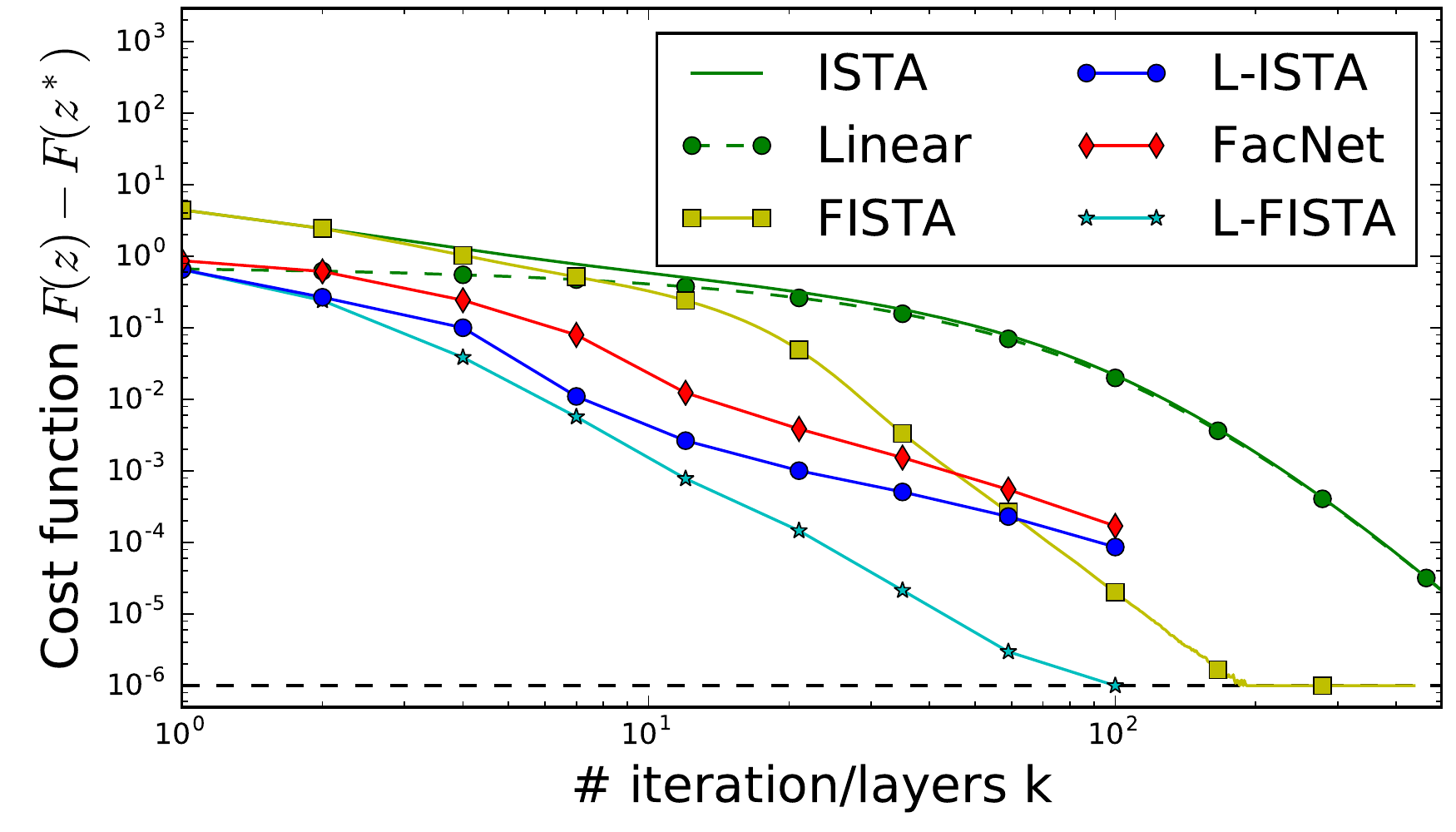}
        \label{fig:layers2}
    \end{subfigure}
    \vspace{-1.3em}
    \caption{
    Evolution of the cost function $F(z_k)-F(z^*)$ with the number of layers or the number of iteration $k$ for different sparsity level.
    \emph{(left)} $\rho=\nicefrac{1}{20}$ and \emph{(right)}$\rho=\nicefrac{1}{4}~.$
    }
    \label{fig:layers}
\end{figure}


\textbf{Gaussian dictionary}
In order to disentangle the role of dictionary structure from the role of data distribution structure, the minimization problem is tested using a synthetic generative model with no structure in the weights distribution.
First, $m$ atoms $d_i \in \R^n$ are drawn \emph{iid} from a multivariate Gaussian with mean $\textbf{0}$ and covariance \textbf I$_n$ and the dictionary $D$ is defined as
$\left(\nicefrac{d_i}{\|d_i\|_2}\right)_{i=1\dots m}~.$
The data points are generated from its sparse codes following a Bernoulli-Gaussian model.
The coefficients $z=(z_1, \dots, z_m)$ are constructed with $z_i = b_i a_i$, where $b_i\sim \mathcal B(\rho)$ and $a_i \sim \mathcal N(0, \sigma \textbf I_m)~,$ where $\rho$ controls the sparsity of the data.
The values are set to $ \smeq m=100,~\smeq n = 64$ for the dictionary dimension, $ \rho = \nicefrac{5}{m}$ for the sparsity level and $\smeq\sigma = 10$ for the activation coefficient generation parameters.
The sparsity regularization is set to $\smeq\lambda = 0.01$.
The batches used for the training are generated with the model at each step and the cost function is evaluated over a fixed test set, not used in the training.

\autoref{fig:layers} displays the cost performance for methods ISTA/FISTA/Linear relatively to their iterations and for methods LISTA/LFISTA/FacNet relatively to the number of layers used to solve our generated problem.
Linear has performances comparable to learned methods with the first iteration but a gap appears as the number of layers increases, until a point where it achieves the same performances as non adaptive methods.
This highlights that the adaptation is possible in the subsequent layers of the networks, going farther than choosing a suitable starting point for iterative methods.
The first layers permit to achieve a large gain over the classical optimization strategy, by leveraging the structure of the problem.
This appears even with no structure in the sparsity patterns of input data, in accordance with the results in the previous section.
We also observe diminishing returns as the number of layers increases.
This results from the phase transition described in \autoref{sub:phase}, as the last layers behave as ISTA steps and do not speed up the convergence.
The 3 learned algorithms are always performing at least as well as their classical counterpart, as it was stated in \autoref{thm:conv_rate}.
We also explored the effect of the sparsity level in the training and learning of adaptive networks.
In the denser setting, the arbitrage between the $\ell_1$-norm and the squared error is easier as the solution has a lot of non zero coefficients.
Thus in this setting, the approximate method is more precise than in the very sparse setting where the approximation must perform a fine selection of the coefficients.
But it also yield lower gain at the beggining as the sparser solution can move faster.

There is a small gap between LISTA and FacNet in this setup.
This can be explained from the extra constraints on the weights that we impose in the FacNet, which effectively reduce the parameter space by half.
Also, we implement the unitary constraints on the matrix $A$  by a soft regularization (see \autoref{eq:facnet}), involving an extra hyper-parameter $\mu$  that also contributes to the small performance gap.
In any case, these experiments show that our analysis accounts for most of the acceleration provided by LISTA, as the performance of both methods are similar, up to optimization errors.

\begin{figure}[t]
\centering
	\begin{subfigure}[t]{0.48\textwidth}
		\centering
		\includegraphics[width=\textwidth]{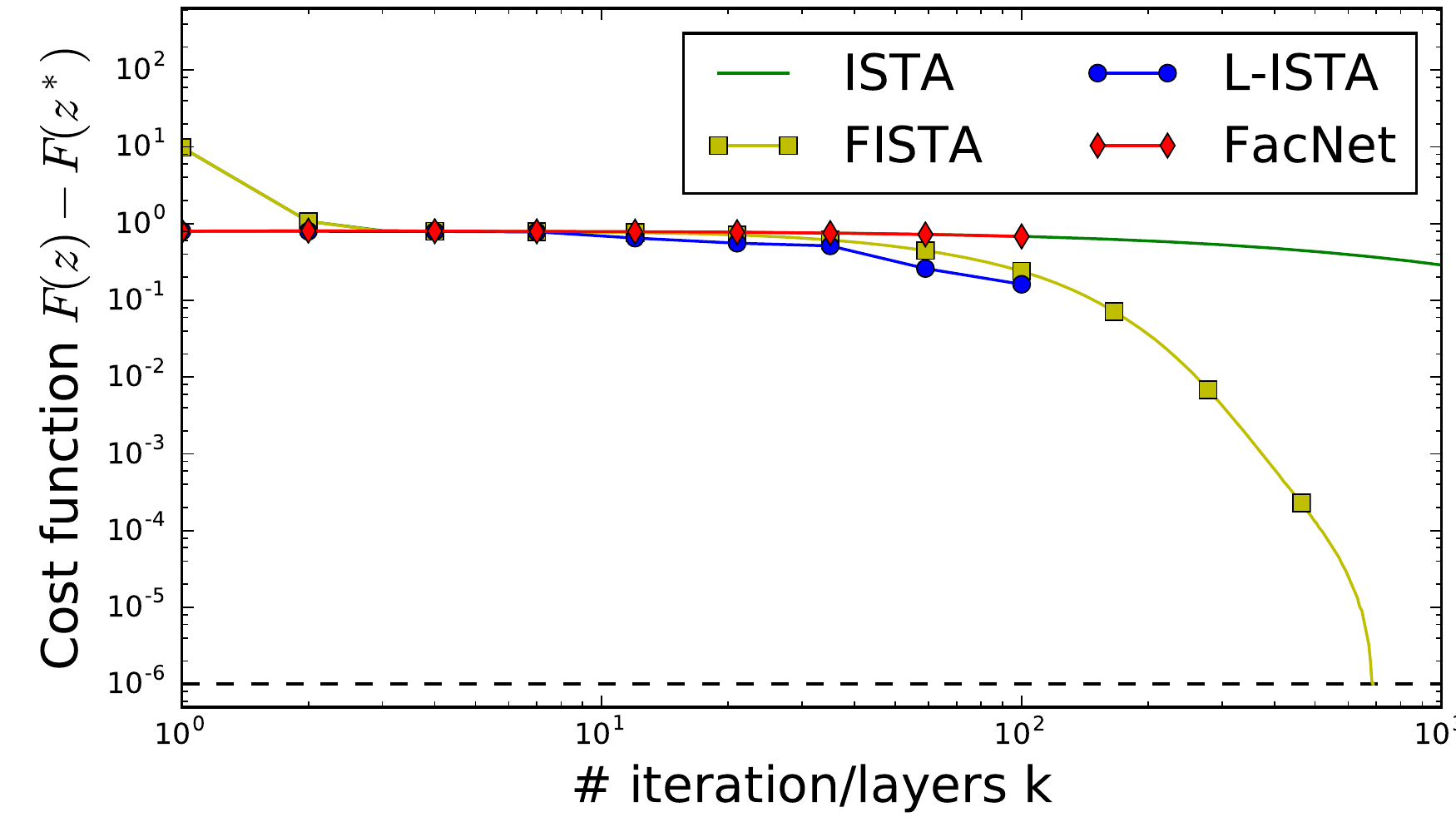}
    \end{subfigure}
	\begin{subfigure}[t]{0.48\textwidth}
		\centering
		\includegraphics[width=\textwidth]{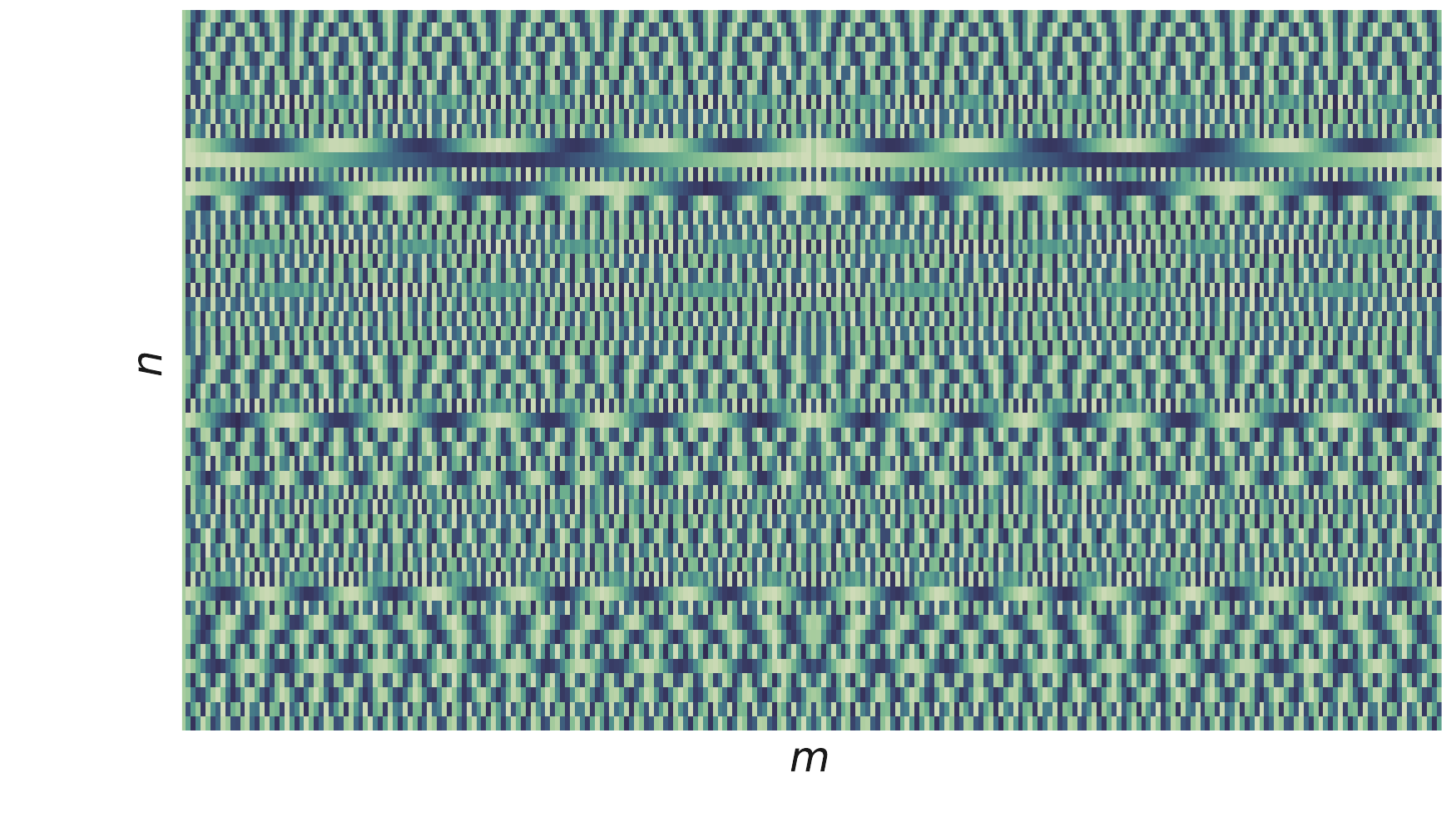}
		
    \end{subfigure}
    \caption{Evolution of the cost function $F(z_k)-F(z^*)$ with the number of layers or the number of iteration $k$ for a problem generated with an adversarial dictionary.}
    \label{fig:adverse}
\end{figure}

\textbf{Adversarial dictionary}
The results from \autoref{sec:math} show that problems with a gram matrix composed of large eigenvalues associated to non sparse eigenvectors are harder to accelerate.
Indeed, it is not possible in this case to find a quasi diagonalization of the matrix $B$ that does not distort the $\ell_1$ norm.
It is possible to generate such a dictionary using Harmonic Analysis.
The Discrete Fourier Transform (DFT) distorts a lot the $\ell_1$ ball, since a very sparse vector in the temporal space is transformed in widely spread spectrum in the Fourier domain.
We can thus design a dictionary for which LISTA and FacNet performances should be degraded.
$ D = \left(\nicefrac{d_i}{\|d_i\|_2}\right)_{i=1\dots m}~$ is constructed such that $ d_{j,k} = e^{-2\pi i j \zeta_k}$, with $\left(\zeta_k\right)_{k \le n}$ randomly selected from $\left\{\nicefrac{1}{m},\dots, \nicefrac{\nicefrac{m}{2}}{m}\right\}$ without replacement.

The resulting performances are reported in \autoref{fig:adverse}.
The first layer provides a big gain by changing the starting point of the iterative methods.
It realizes an arbitrage of the tradeoff between starting from $\bf 0$ and starting from $y~.$
But the next layers do not yield any extra gain compared to the original ISTA algorithm.
After $4$ layers, the cost performance of both adaptive methods and ISTA are equivalent.
It is clear that in this case, FacNet does not accelerate efficiently the sparse coding, in accordance with our result from \autoref{sec:math}.
LISTA also displays poor performances in this setting.
This provides further evidence that FacNet and LISTA share the same acceleration mechanism as adversarial dictionaries for FacNet are also adversarial for LISTA.


 

%
%


%


\subsection{Sparse coding with over complete dictionary on images}
\label{sub:img}



\textbf{Wavelet encoding for natural images}
\label{par:wavelet}
A highly structured dictionary composed of translation invariant Haar wavelets is used to encode 8x8 patches of images from the PASCAL VOC 2008 dataset.
The network is used to learn an efficient sparse coder for natural images over this family.
$500$ images are sampled from dataset to train the encoder.
Training batches are obtained by uniformly sampling patches from the training image set to feed the stochastic optimization of the network.
The encoder is then tested with $10000$ patches sampled from $100$ new images from the same dataset.


\textbf{Learned dictionary for MNIST}
To evaluate the performance of LISTA for dictionary learning, LISTA was used to encode MNIST images over an unconstrained dictionary, learned {\it a priori} using classical dictionary learning techniques.
The dictionary of $100$ atoms was learned from $10000$ MNIST images in grayscale rescaled to 17x17 using the implementation of \cite{Mairal2009a} proposed in scikit-learn, with $\lambda = 0.05$.
Then, the networks were trained through backpropagation using all the $60000$ images from the training set of MNIST.
Finally, the perfornance of these encoders were evaluated with the $10000$ images of the training set of MNIST.

\begin{figure}[t]
\centering
    \begin{subfigure}[t]{0.48\textwidth}
		\includegraphics[width=\textwidth]{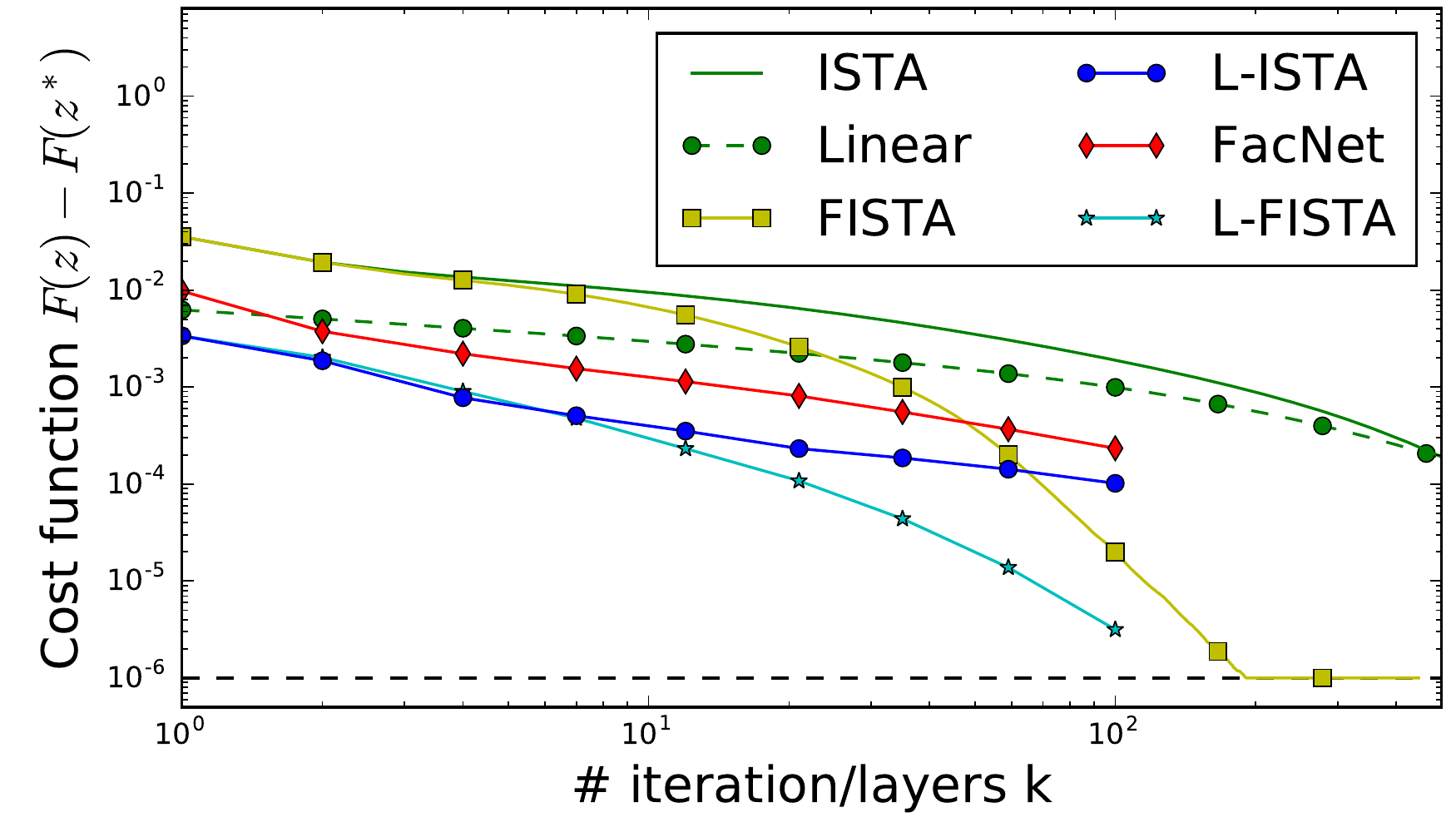}
		\caption{Pascal VOC 2008}
		\label{fig:pascal}
    \end{subfigure}
    \begin{subfigure}[t]{0.48\textwidth}
        \includegraphics[width=\textwidth]{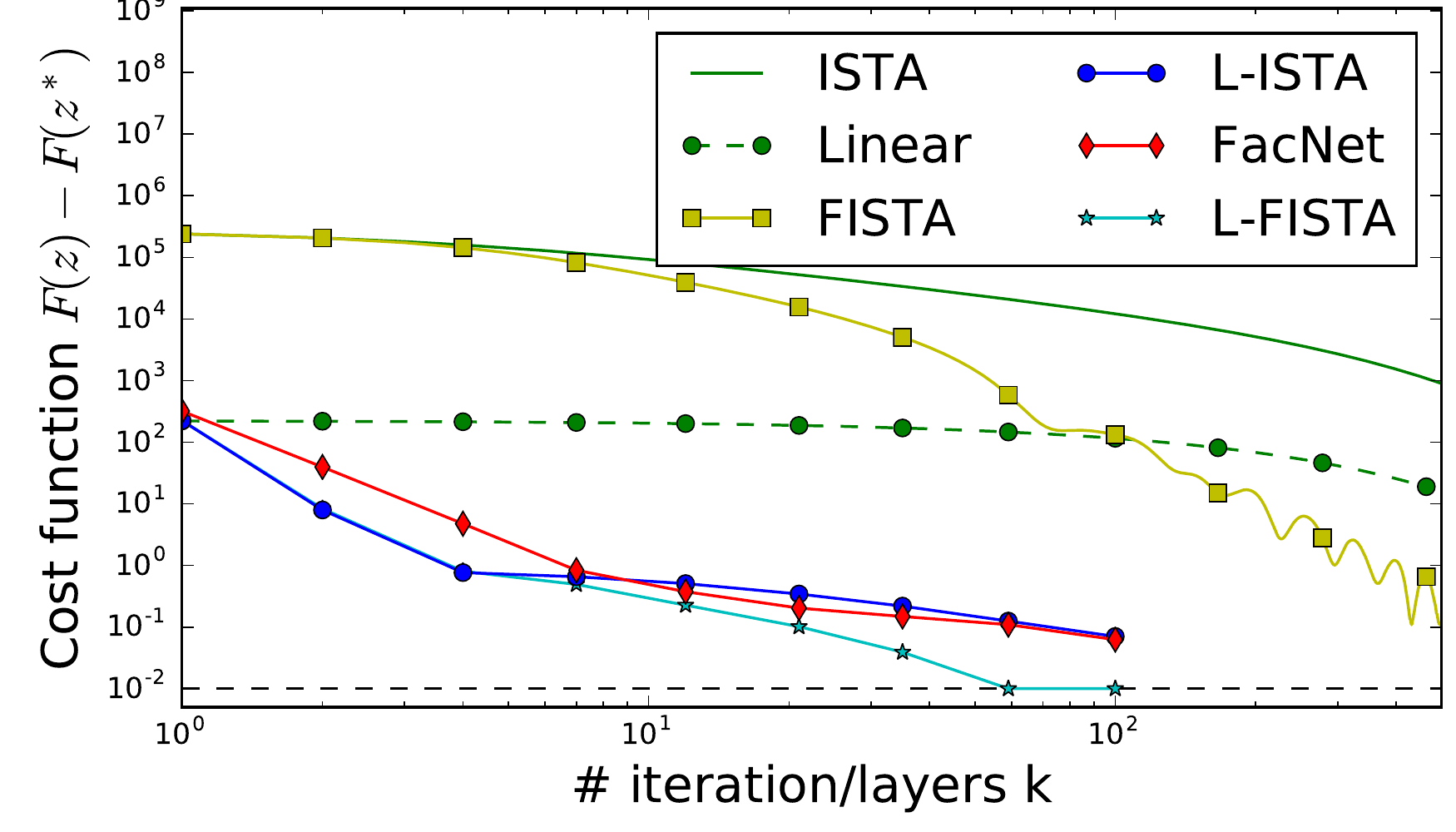}
        \caption{MNIST}
        \label{fig:mnist}
    \end{subfigure}
    \vspace{-.5em}
    \caption{
    Evolution of the cost function $F(z_k)-F(z^*)$ with the number of layers or the number of iteration $k$ for two image datasets.}
    \label{fig:images}
\end{figure}

The \autoref{fig:images} displays the cost performance of the adaptive procedures compared to non-adaptive algorithms.
In both scenario, FacNet has performances comparable to the one of LISTA and their behavior are in accordance with the theory developed in \autoref{sec:math}.
The gains become smaller for each added layer and the initial gain is achieved for dictionary either structured or unstructured.
The MNIST case presents a much larger gain compare to the experiment with natural images.
This results from the difference of structure of the input distribution, as the MNIST digits are much more constrained than patches from natural images and the network is able to leverage it to find a better encoder.
In the MNIST case, a network composed of $12$ layers is sufficient to achieve performance comparable to ISTA with more than $1000$ iterations.



\section{Conclusions}
\label{sec:conclusions}


In this paper we studied the problem of finite computational budget approximation of sparse coding.
Inspired by the ability of neural networks to accelerate over splitting methods on the first few iterations, we have studied which properties of the dictionary matrix and the data distribution lead to such acceleration.
Our analysis reveals that one can obtain acceleration by finding approximate matrix factorizations of the dictionary which nearly diagonalize its Gram matrix, but whose orthogonal transformations leave approximately invariant the $\ell_1$ ball.
By appropriately balancing these two conditions, we show that the resulting rotated proximal splitting scheme has an upper bound which improves over the ISTA upper bound under appropriate sparsity.

In order to relate this specific factorization property to the actual LISTA algorithm, we have introduced a reparametrization of the neural network that specifically computes the factorization, and incidentally provides reduced learning complexity (less parameters) from the original LISTA.
Numerical experiments of \autoref{sec:exp} show that such reparametrization recovers the same gains as the original neural network, providing evidence that our theoretical analysis is partially explaining the behavior of the LISTA neural network.
Our acceleration scheme is inherently transient, in the sense that once the iterates are sufficiently close to the optimum, the factorization is not effective anymore.
This transient effect is also consistent with the performance observed numerically, although the possibility remains open to find alternative models that further exploit the particular structure of the sparse coding.
Finally, we provide evidence that successful matrix factorization is not only sufficient but also necessary for acceleration, by showing that Fourier dictionaries are not accelerated.

Despite these initial results, a lot remains to be understood on the general question of optimal tradeoffs between computational budget and statistical accuracy.
Our analysis so far did not take into account any probabilistic consideration (e.g. obtain approximations that hold with high probability or in expectation).
Another area of further study is the extension of our analysis to the FISTA case, and more generally to other inference tasks that are currently solved via iterative procedures compatible with neural network parametrizations, such as inference in Graphical Models using Belief Propagation or other ill-posed inverse problems.

\bibliographystyle{\Bibstyle}
\bibliography{\BibFile}

\appendix
\crossref{}
\appendix

\section{Learned Fista}
\label{sec:lfista}

A similar algorithm can be derived from FISTA, the accelerated version of ISTA to obtain LFISTA (see \autoref{fig:lfista}~).
The architecture is very similar to LISTA, now with two memory taps:
It introduces a momentum term to improve the convergence rate of ISTA as follows:
\begin{enumerate}
	\item $\displaystyle y_k = z_{k} + \frac{\smeq t_{k-1} - 1}{t_k}(z_k - z_{k-1})$~,
	\item $\displaystyle z_{k+1} = h_{\frac{\lambda}{L}}\left(y_k - \frac{1}{L}\nabla E(y_k)\right)= h_{\frac{\lambda}{L}}\left(({\bf I} - \frac{1}{L}B)y_k + \frac{1}{L}D\tran x\right)$~,
	\item $ \displaystyle t_{k+1} = \frac{1 + \sqrt{1 + 4t_k^2}}{2}$~.
\end{enumerate}
By substituting the expression for $y_k$ into the first equation, we obtain a generic recurrent architecture very similar to LISTA, now with two memory taps, that we denote by LFISTA:
\[z_{k+1} = h_{\theta}( W_g^{(k)} z_k + W_m^{(k)} z_{k-1} + W_e^{(k)} x)~.\]
This model is equivalent to running $K$-steps of FISTA when its parameters are initialized with
\begin{eqnarray}
	W_g^{(k)} & = & \left(1 + \frac{t_{k-1} - 1}{t_{k}}\right)\left({\bf I} - \frac{1}{L}B\right)~, \nonumber \\
	W_m^{(k)} & = & \left(\frac{1 - t_{k-1}}{t_{k}}\right)\left({\bf I} - \frac{1}{L}B\right)~, \nonumber \\
	W_e^{(k)} & = & \frac{1}{L}D\tran~. \nonumber
\end{eqnarray}
The parameters of this new architecture, presented in \autoref{fig:lfista}~, are trained analogously as in the LISTA case.

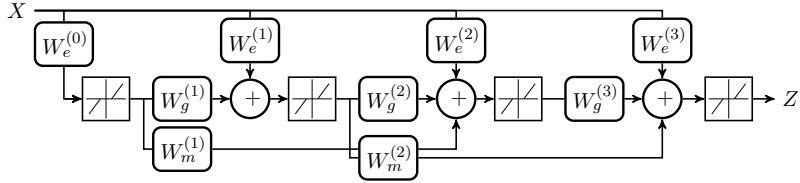
\begin{figure}[t]
\centering
\scalebox{.75}{

\begin{tikzpicture}[scale=\tkzscl]

\tikzset{
    >=stealth',
    varstyle/.style={
           rectangle,
           rounded corners,
           draw=black, very thick,
           fill=white,
           text width=2em,
           minimum height=2em,
           text centered},
    op/.style={
           circle,
           draw=black, very thick,
           fill=white,
           text width=.5em,
           minimum height=.3em,
           text centered},
    st/.style={
           draw,
           thick,
           rectangle, 
           fill=white,
           text width=1.5em, 
           minimum height=2em,},
    pil/.style={
           ->,
           thick,
    }
};

\def\nlayer{3}

\node (X) {$X$};
\node[right=1em of X] (linear) {};
\node[inner sep=0,minimum size=0,below=3.7em of linear] (p) {};
\node[varstyle, above=1.45em of p] (B) {$W_e^{(0)}$};
\foreach \i in {1,...,\nlayer}{
    \node[st, right=.8em of p] (sta\i) {};
    \path (p.east) edge[pil] (sta\i.west);

    \node[varstyle, right=1em of sta\i] (S) {$W_g^{(\i)}$};
    \node[op, right=.8em of S] (p) {$+$};
	"\ifthenelse{\i>1}{\draw[pil] (S2.east) -| (p.south);}{}"
    "\ifthenelse{\i<\nlayer}{\node[varstyle, below=(2*(\i-1)+1)*.2em of S] (S2) {$W_m^{(\i)}$};}{}";
    \node[inner sep=0, minimum size=0, right=.6em of sta\i] (branch) {};
    \node[varstyle, above=.7em of p] (B\i) {$W_e^{(\i)}$};
    
    \draw ($ (sta\i.center) - (0,.9em)  $) -- ($ (sta\i.center) + (0,.9em)  $);
    \draw ($ (sta\i.center) - (.35,0em)  $) -- ($ (sta\i.center) + (.35,0em)  $);
    \draw ($ (sta\i.center) - (.1,0em)  $) -- ($ (sta\i.center) - (.3,.7em)  $);
    \draw ($ (sta\i.center) + (.1,0em)  $) -- ($ (sta\i.center) + (.3,.7em)  $);
    \draw[pil] (sta\i.east) -- (S) -- (p.west);
    \draw[pil] (X) -| (B\i) -- (p.north);
	"\ifthenelse{\i<\nlayer}{\draw[thick] (branch) |- (S2.west);}{}"
}
\draw[pil] (linear.center) -- (B) |- (sta1.west);

\node[st, right=1em of p] (sta) {};
    \draw ($ (sta.center) - (0,.9em)  $) -- ($ (sta.center) + (0,.9em)  $);
    \draw ($ (sta.center) - (.35,0em)  $) -- ($ (sta.center) + (.35,0em)  $);
    \draw ($ (sta.center) - (.1,0em)  $) -- ($ (sta.center) - (.3,.7em)  $);
    \draw ($ (sta.center) + (.1,0em)  $) -- ($ (sta.center) + (.3,.7em)  $);
\node[right=1em of sta] (Z) {$Z$};
\path (p.east) edge[pil] (sta.west);
\path (sta.east) edge[pil] (Z.west);

\end{tikzpicture}
}
\caption{
Network architecture for LFISTA.
This network is trainable through backpropagation and permits to approximate the sparse coding solution efficiently.}
\label{fig:lfista}
\end{figure}

\section{Proofs for the convergence rate using LISTA}
\label{sec:proof_conv}

\begin{propositionA}
\label{prop:bound_error}
Suppose that $R = A\tran S A - B$ is positive definite, and define 
\begin{equation}
\label{eq:algo_a}
z_{k+1} = \arg\min_z \widetilde{F}(z, z_k)~,\text{ and }
\end{equation}
$\delta_A(z) = \|A z \|_1 - \|z \|_1$.
Then we have
\begin{equation}
\label{eq:bo2_a}
F(z_{k+1}) - F(z^*) \leq \frac{1}{2}  \left( (z^* -z_{k})\tran R(z^* -z_{k})  - (z^* -z_{k+1})\tran R(z^* -z_{k+1}) \right) + \langle \partial \delta_A(z_{k+1}) ,  z_{k+1} - z^* \rangle ~.
\end{equation}
\end{propositionA}

\begin{proof}
We define
$$f(t) = F\left( t z_{k+1} + (1-t) z^*\right)~,~t \in [0,1]~.$$
Since $F$ is convex, $f$ is also convex in $[0,1]$.
Since $f(0) = F(z^*)$ is the global minimum, 
it results that $f'(t)$ is increasing in $(0,1]$, and hence
\begin{equation*}
F(z_{k+1}) - F(z^*) = f(1) - f(0) = \int f'(t) dt \leq f'(1)~,
\end{equation*}
where $f'(1)$ is any element of  $\partial f(1)$.
Since $\delta_A(z)$ is a difference of convex functions, its subgradient 
can be defined as a limit of infimal convolutions \cite{Hiriart-Urruty1991}.
We have
$$\partial f(1) = \langle \partial F(z_{k+1}), z_{k+1} - z^* \rangle~,$$
and since
$$\partial F(z) = \partial \widetilde{F}(z,z_k) - R (z - z_k) - \partial \delta_A(z)~\text{ and}~0 \in \partial \widetilde{F}(z_{k+1}, z_k)$$
it results that
\begin{equation*}
\partial F(z_{k+1}) = -R (z_{k+1} - z_k) - \partial \delta_A(z_{k+1})~,
\end{equation*}
and thus
\begin{equation}
F(z_{k+1}) - F(z^*) \leq (z^*-z_{k+1})\tran R ( z_{k+1} - z_k) + \langle \partial \delta_A(z_{k+1}) , (z^*-z_{k+1}) \rangle~.
\end{equation}
\autoref{eq:bo2} is obtained by observing that
\begin{equation}
\label{eq:bt1_a}
(z^*-z_{k+1})\tran R ( z_{k+1} - z_k) \leq \frac{1}{2} \left( (z^* -z_{k})\tran R(z^* -z_{k})  - (z^* -z_{k+1})\tran R(z^* -z_{k+1}) \right) ~,
\end{equation}
thanks to the fact that $R \succ 0$.
\end{proof}

\restatethm{\convRate}{thm:conv_rate}

\begin{proof} The proof is adapted from \citep[Theorem 3.1]{Beck2009}.\\
From \autoref{prop:bound_error},
we start by using \autoref{eq:bo2_a} to bound terms of the form $F(z_n) - F(z^*)$:
{\small 
\begin{equation*}
F(z_n) - F(z^*) \leq \langle \partial\delta_{A_n}(z_{n+1}) , (z^*-z_{n+1}) \rangle + \frac{1}{2} \left( (z^* -z_{n})\tran R_n (z^* -z_{n})  - (z^* -z_{n+1})\tran R_n(z^* -z_{n+1}) \right)~.
\end{equation*}
}
Adding these inequalities for $n = 0 \dots k-1$ we obtain
\begin{eqnarray}
\label{eq:vg1_a}
\left(\sum_{n=0}^{k-1} F(z_n) \right) - k F(z^*) &\leq& \sum_{n=0}^{k-1}  \langle \partial \delta_{A_n}(z_{n+1}) , (z^*-z_{n+1}) \rangle + \\ \nonumber
&& + \frac{1}{2} \left( (z^*- z_0)\tran R_0 (z^* - z_0) - (z^*- z_k)\tran R_{k-1} (z^* - z_k) \right) + \\ \nonumber 
&& + \frac{1}{2} \sum_{n=1}^{k-1} (z^* - z_n)\tran ( R_{n-1} - R_{n}) (z^* - z_n)~.
\end{eqnarray}
On the other hand, we also have
\begin{eqnarray*}
\label{eq:vg2_a}
F(z_n) - F(z_{n+1}) &\geq& F(z_n) - \tilde{F}(z_n,z_n) + \tilde{F}(z_{n+1},z_n) - F(z_{n+1}) \\
&=& - \delta_{A_n}(z_n) + \delta_{A_n}(z_{n+1}) + \frac{1}{2}(z_{n+1}-z_n)\tran R_n (z_{n+1}-z_n)~,
\end{eqnarray*}
which results in 
\begin{eqnarray}
\label{eq:vg3_a}
\sum_{n=0}^{k-1} (n+1)(F(z_n) - F(z_{n+1})) &\geq& \frac{1}{2}\sum_{n=0}^{k-1}(n+1) (z_{n+1}-z_n)\tran R_n (z_{n+1}-z_n) + \\ \nonumber
&& + \sum_{n=0}^{k-1}(n+1) \left( \delta_{A_n}(z_{n+1}) - \delta_{A_n}(z_n) \right) \\ \nonumber 
\left(\sum_{n=0}^{k-1} F(z_n) \right) - k F(z_k) &\geq&  \sum_{n=0}^{k-1} (n+1)\left(\frac{1}{2}(z_{n+1}-z_n)\tran R_n (z_{n+1}-z_n) + \delta_{A_n}(z_{n+1}) - \delta_{A_n}(z_n) \right) ~.
\end{eqnarray}
Combining \autoref{eq:vg1_a} and \autoref{eq:vg3_a} we obtain
\begin{eqnarray}
F(z_k) - F(z^*) &\leq& \frac{(z^*- z_0)\tran R_0 (z^* - z_0) + 2\langle \nabla \delta_{A_0}(z_{1}) , (z^*-z_{1}) \rangle }{2k} + \frac{\alpha - \beta}{2k}\\ \nonumber
\end{eqnarray}
with
$$ \alpha = \sum_{n=1}^{k-1} \left(  2\langle \nabla \delta_{A_n}(z_{n+1}) , (z^*-z_{n+1}) \rangle + (z^* - z_n)\tran ( R_{n-1} - R_{n}) (z^* - z_n) \right)~,$$
$$\beta = \sum_{n=0}^{k-1} (n+1)\left((z_{n+1}-z_n)\tran R_n (z_{n+1}-z_n) + 2\delta_{A_n}(z_{n+1}) - 2\delta_{A_n}(z_n) \right)~.$$
\end{proof}

\restatethm{\oneStep}{corr:one_step}
\begin{proof}
	 We verify that in that case, $R_{n-1} - R_n \equiv 0$ and for $n>1$ and $\delta_{A_n} \equiv 0$ for $n > 0~.$
\end{proof}

\crossref{}
\appendix

\section{Existence of a gap for generic dictionary.}

%
%
%

\subsection{Properties of $\mathcal E_\delta$ }
\label{sub:prop_A}

\begin{propositionA}
\label{prop:quasi_unitary}
If a matrix $A$ has its columns in $\mathcal
E_{\delta,i}$, then it is almost unitary for small
value of $\delta$. More precisely, denoting
$\nu = A\tran A-\I_K$, when $\delta \to 0$ 
\[
	\|\nu\|_F = \bO\delta
\]
\end{propositionA}

\begin{proof}
Let $\nu = A\tran A - \I_K$. As $A_i$ are in $\mathcal E_{\delta, i}$, 
\[ \nu_{i,i} = A_i\tran A_i  - 1 = 0 \] 
We can verify that for $i \neq j$
\begin{align*}
\nu_{i, j} = A_i\tran A_j
	& = \delta\sqrt{1-\delta^2} (e_i\tran h_j
		+ e_j\tran h_i)
		+ \delta^2 h_i\tran h_j\\
	& = \delta(e_i\tran h_j + e_j\tran h_i)
		+ \bO{\delta^2}
\end{align*}
This permits to bound the Frobenius norm of $\nu$ \ie
\[
	\|\nu\|_F^2 = \sum_{1 \le i, j \le K}
	\nu_{i,j}^2 = \delta^2 \sum_{\substack{1\le i,j \le K\\
	i\neq j}} (e_i\tran h_j + e_j\tran h_i)^2 + \bO{\delta^3}~.
\]

\begin{align*}
	\|\nu\|_F^2 = \sum_{1 \le i, j \le K}
	\nu_{i,j}^2 & = \delta^2 \sum_{\substack{1\le i,j \le K\\
	i\neq j}} (e_i\tran h_j + e_j\tran h_i)^2 + \bO{\delta^3}~,\\
	& = \delta^2 \sum_{1 \le i, j \le K} (h_{i,j} + h_{j,i})^2 + \bO{\delta^3}~,\\
	& = \delta^2 \|H + H\tran\|_F^2 + \bO{\delta^3}~,\\
	& = 4\delta^2 \|H\|_F^2 + \bO{\delta^3} = 4\delta^2 K + \bO{\delta^3}~.
	\tag*{as $\|h_i\|_2^2 = 1$ }
\end{align*}
\end{proof}


\begin{propositionA}
\label{prop:rotation}
For $A \subset \mathcal E_\delta$, and for any symmetric matrix
$U \in \R^{K\times K}$, when $\delta \to 0$,
\[
	\left\|A^{-1}UA\right\|_F^2  \le \|U\|_F^2
	+ \bO{\delta^3}\\
\]
\end{propositionA}

\begin{proof}
For $U \in \R^{K\times K}$ symmetric, as $A$ is quasi unitary, 
\begin{align*}
 \left\|A^{-1}UA\right\|_F^2 
 	& = \Tr{A\tran U (A^{-1})\tran A^{-1}UA}
 	  = \Tr{U (A\tran A)^{-1}UAA\tran}\\
 	& = \Tr{U (\I_K + \nu)^{-1}U(\I_K + \nu)}
 	  = \Tr{U (\I_K - \nu + \nu^2 + \bO{\nu^3})U
 	  	(\I_K + \nu)}\\
 	& = \Tr{U U + U\nu^2U - U\nu U\nu + \bO{\nu^3})}\\
	& = \|U\|_F^2 + \|\nu U\|_F^2 -
		\|\nu^{\frac{1}{2}}U \nu^{\frac{1}{2}} \|_F^2
		+ \bO{\|\nu^{3/2}\|_F^2}
\end{align*}

Notice that
\[
	\|\nu^{\frac{1}{2}}U \nu^{\frac{1}{2}} \|_F
	= \|(U\nu)^{\frac{1}{2}}{}\tran (U\nu)^{\frac{1}{2}}\|_F
	= \|(U\nu)^{\frac{1}{2}}\|_F^2
	\ge \|U\nu\|_F
\]
Thus $\|\nu U\|_F^2 - \|\nu^{\frac{1}{2}}U \nu^{\frac{1}{2}} \|_F^2 \le 0$
and by submultiplicativity of $\|\cdot\|_F^2$,
\[
	\|\nu^{3/2}\|_F^2 \le \|\nu\|_F^3 = \bO{\delta^3}
		~~~\Rightarrow~~~ \bO{\|\nu^{3/2}\|_F^2} = \bO{\delta^3}~.
\]
By combining all these results, we get:
\[
	\left\|A^{-1}UA\right\|_F^2 \le  \|U\|_F^2 +\bO{\delta^3}\\
\]
\end{proof}

\begin{propositionA}
\label{lem:A-1_AT}
For a matrix $A \subset \mathcal E_{\delta}$ and any
matrices $X, Y \in \R^{K\times K}$, when
$\delta \to 0~,$
\[
\left\|A^{-1}XA - Y\right\|_F^2
	 \le \left\|X - AYA\tran\right\|_F^2
		+ \left\|Y\right\|_F^2\left\|\nu\right\|_F^2
		 + \bO{\delta^3}~.
\]
\end{propositionA}

\begin{proof}
First, we split the error of replacing $\left\|
A^{-1}XA - Y\right\|_F^2$ by $\left\|X -
AYA\tran\right\|_F^2$ in two terms. Both are linked
to the quasi unitarity of $A$. The first term arises
as we replace $A^{-1}$ by $A\tran$,
\begin{align*}
\left\|A^{-1}XA - Y\right\|_F^2
	& = \left\|A^{-1}\left(X - AYA^{-1}\right)A\right\|_F^2
	 = \left\|A^{-1}\left(X - AY\underbrace{(A\tran A - \nu)}_{\I_K}
	 	A^{-1}\right)A\right\|_F^2\\
	& = \left\|A^{-1}\left(X - AYA\tran + AY\nu A^{-1}\right)A\right\|_F^2\\
	& \le 2\left\|A^{-1}\left(X - AYA\tran\right)A\right\|_F^2
		+ 2\left\|A^{-1}\left(AY\nu A^{-1}\right)A\right\|_F^2
		\tag*{\scriptsize(triangular inequality)}\\
	& \le 2\left\|X - AYA\tran\right\|_F^2
		+ 2\left\|Y\nu\right\|_F^2 + \bO{\delta^3}
		\tag*{\scriptsize(\autoref{prop:rotation})}\\
	& \le 2\left\|X - AYA\tran\right\|_F^2
		+ 2\left\|Y\right\|_F^2\left\|\nu\right\|_F^2
		+ \bO{\delta^3}
		\tag*{\scriptsize(submultiplicativity $\|\cdot\|_F$ )}\\
	& \le 2\left\|X - AYA\tran\right\|_F^2
		+ 8\delta^2K\left\|Y\right\|_F^2
		+ \bO{\delta^3}
		\tag*{\scriptsize(\autoref{prop:quasi_unitary})}
\end{align*}
\end{proof}

\begin{propositionA}
\label{prop:unitary}
For $A \subset \mathcal E_\delta$, and for any matrix
$U \in \R^{K\times K}$, when $\delta \to 0$,
\[
\left\|A\tran UA\right\|_F^2
	= \|U\|_F^2\left(1 + \|\nu\|_F^2\right)
	  +\bO{\delta^3}
\]
\end{propositionA}
\begin{proof}
	
\begin{align*}
\|AUA\tran\|_F = \|AXA\tran AA^{-1}\|_F^2
	& = \|AU(\I_K+\nu)A^{-1}\|_F^2\\
    & = \|U+U\nu\|_F^2 + \bO{\delta^3}
    \tag*{\scriptsize(\autoref{prop:rotation})}\\
    & = 2\|U\|_F^2 + 2\|U\nu\|_F^2 + \bO{\delta^3}
    \tag*{\scriptsize(triangular inequality)}\\
    & = 2\|U\|_F^2 + 2\|U\|_F^2\|\nu\|_F^2
    	+\bO{\delta^3}
    \tag*{\scriptsize($\|\cdot\|_F$ is
    	  sub multiplicative)}\\
	& \le 2\|U\|_F^2 + 8\delta^2K\|U\|_F^2
		+ \bO{\delta^3}
		\tag*{\scriptsize(\autoref{prop:quasi_unitary})}
\end{align*}
\end{proof}

%
%
%

\subsection{Control the deviation of $\|\cdot\|_B$ }
\label{sub:control_Rf} 

\restatethm{\controlR}{lem:control_Rf}

\begin{proof}
\label{proof:control_Rf}
First, we use the results from \autoref{lem:A-1_AT}
to remove the inverse matrix $A^{-1}$
\[
\left\|A^{-1}SA - B\right\|_F^2 
\le \left\|S - ABA\tran\right\|_F^2
	+\left\|B\right\|_F^2\left\|\nu\right\|_F^2
	+ \bO{\delta^3}~.
\]
Using \autoref{prop:quasi_unitary} with
$A \subset \mathcal E_\delta$,
\[
\|\nu\|_F^2 = 4\delta^2K + \bO{\delta^3}
\]
and
\[
\left\|A^{-1}SA - B\right\|_F^2 
    \le \left\|S - ABA\tran\right\|_F^2
    	+4\delta^2K\|B\|_F^2 + \bO{\delta^3}~.
\] 
Then, we only need to control $\left\|S -
ABA\tran\right\|_F^2$.
First we note that this can be split into 2 terms
\begin{align*}
\left\|S - ABA\tran\right\|_F^2 
&= \sum_{i=1}^K\sum_{\substack{j=1\\j \neq i}}^K
	(A_i\tran B A_j)^2
 = \sum_{i=1}^K\sum_{j=1}^K (A_i\tran B A_j)^2
	- \sum_{i=1}^K (A_i\tran B A_i)^2\\
&= \left\|ABA\tran\right\|_F^2
	- \sum_{i=1}^K \|DA_i\|_2^4\\
& = \left\|B\right\|_F^2(1+4\delta^2K)
	- \sum_{i=1}^K \|DA_i\|_2^4 + \bO{\delta^3}
	\tag*{\scriptsize(\autoref{prop:unitary})}
\end{align*}
	
The first term is the squared Frobenius norm of
a Wishart matrix and can be controlled by
\begin{align*}
\E[D]{\|B\|_F^2}
&= \E[D]{\sum_{i=1}^K\sum_{j=1}^K B_{i,j}^2}
 = \sum_{i=1}^K\sum_{j=1}^K
   \E[D]{\left(\sum_{l=1}^p
 		 d_{i,k}d_{j,k}\right)^2}\\
&= \sum_{i=1}^K\sum_{\substack{j=1\\i\neq j}}^K
   \E[D]{\left(\sum_{l=1}^p
		 d_{i,l}d_{j,l}\right)^2}
 + \sum_{i=1}^K \E[D]{\|d_i\|_2^4}\\
&=  \sum_{i=1}^K\sum_{\substack{j=1\\i\neq j}}^K
	\left[\sum_{l=1}^p\E[D]{
		d_{j,l}^2d_{i,l}^2} 
	+ \sum_{l=1}^p\sum_{\substack{m=1\\m\neq l}}^p
\underbrace{\E[D]{d_{i,l}d_{i,m}d_{j,l}d_{j,m}}}
			_{=0}\right] + K\\
&= \sum_{i=1}^K\sum_{\substack{j=1\\i\neq j}}^K 
   \sum_{l=1}^p\E[d_j]{d_{j,l}^2}
   \E[d_i]{d_{i,l}^2} + K
   \tag*{($d_i$ are independent)}\\
&=  \sum_{i=1}^K\sum_{\substack{j=1\\i\neq j}}^K
	\sum_{l=1}^p \frac{1}{p^2} + K
 = \frac{K(K-1)}{p} + K~.
 \tag*{\scriptsize($\E[d]{d_{i,j}^2} = \frac{1}{p}$, \citep{Song1997})}
\end{align*}

For the second term, consider $u \in \mathcal
E_{\delta,i}~,$ such that $u = \sqrt{1-\mu^2}e_i +
\mu h$ for $0<\mu<\delta$, $h\in Span(e_i)^\perp$.
Given $i \in [K]$, $Be_i$ can be decomposed as $z_1 e_i + z_2 h_i$,
	with $h_i \in Span(e_i)^\perp \cap \mathcal S^{K-1}$. Using basic algebra,
	$z_1$ and $z_2$ are:
	\begin{align}
		z_1 & =  e_i\tran Be_i = \|De_i\|_2^2 = \|d_i\|_2^2= 1~.\\[.5em]
		z_2^2 & =  \|Be_i\|_2^2 - z_1^2 = \|D\tran d_i\|_2^2 - 1
		\label{eq:Bei_cross}
	\end{align}
	Also, for all $i, j \in [K]$, if $i\neq j$ then
	$h_i\tran e_j = \begin{cases}0 &\text{ if } \|D\tran d_i\|_2^2 = 1\\
								\frac{d_i\tran d_j}{\|D\tran d_i\|_2^2 - 1} & \text{elsewhere}\\
								 \end{cases}$ 
Then  
\begin{align*}
\|D u\|_2^2 & = u\tran D\tran D u = u\tran B u
	\nonumber\\
& = (1-\mu^2)\underbrace{\|D e_i\|_2^2}
	_{\|d_i\|_2 = 1} + \mu^2\|D h\|_2^2
  + 2\mu\sqrt{1-\mu^2}\underbrace{h\tran B e_i}
    _{z_2h\tran h_i}\label{eq:proj_B}
\end{align*}
Thus, with the notation from \autoref{eq:Bei_cross},
$h\tran Be_i = z_2 h\tran h_i$ and
\begin{equation}
	\|Du\|_2^2
		= (1-\delta^2) + \delta^2\|Dh\|_2^2
		+ 2\delta\sqrt{1-\delta^2}
		   \sqrt{\|D\tran d_i\|_2^2-1}h\tran h_i
\end{equation}
Now we can use this to derive a lower bound on
$\max_{u\in \mathcal E_{\delta,i}}\|D\tran u\|_2^2$ 
when $\delta \to 0~,$
\begin{align*}
	\max_{u\in \mathcal E_{\delta, i}}\|D u\|_2^2
	& \ge 1 + 2\delta\sqrt{\|D\tran d_i\|_2^2-1} 
		  +\delta^2\left(\|D\tran d_i\|_2^2-1\right)
		  +\bO{\delta^3}\\
		\tag*{\it$\smeq(u = \hat A_i)$}
\end{align*}
Taking the square of this relation yields
\[
\max_{u\in \mathcal E_{\delta, i}}\|D u\|_2^4
	 \ge 1 + 4\delta\sqrt{\|D\tran d_i\|_2^2-1} 
		  +6\delta^2\left(\|D\tran d_i\|_2^2-1\right)
		  +\bO{\delta^3}\\
\]
Taking the expectation yields
\begin{align*}
\E[D]{\max_{u\in\mathcal E_{\delta, i}}\|D u\|_2^4} 
	& \ge 1 + 4\delta\E[D]{
			\sqrt{\|D\tran d_i\|_2^2-1}}
		+ 6\delta^2\E[D]{\|D\tran d_i\|_2^2-1}
		+\bO{\delta^3}~.
\end{align*}
	
The random variable $pY_i^2 = p(\|D\tran d_i\|_2^2-1)$
are distributed as $\chi_{K-1}^2$. Indeed, the atoms
$d_i$ are uniformly distributed over $\mathcal S^{K-1}
$. As this distribution is rotational invariant,
without loss of generality, we can take $d_i = e_1$.
Then $Y_i$ is simply sum of $K-1$ squared normal
gaussians {\it rv} with variance $\frac{1}{p}$, 
\begin{align*}
	Y_i^2 = \|D\tran e_1\|_2^2-1
	&= \sum_{j = 2}^K d_{j,1}^2~,
\end{align*}
and $\sqrt{p} Y_i$ is distributed as $\chi_{K-1}$.
A lower bound for its expectation is
\[
	\E[D]{Y_i} = \sqrt{\frac{2}{p}}
		\frac{\Gamma\left(\frac{K}{2}\right)}
			 {\Gamma\left(\frac{K-1}{2}\right)}
		\ge \frac{K-1}{\sqrt{pK}}
	~~~~\text{and}~~~~\E[D]{Y_i^2} = \frac{K-1}{p}
\]
We derive a lower bound for the second term when
$\delta \to 0~$,
\begin{align*}
\E[D]{\max_{u\in\mathcal E_{\delta, i}}
	\|D\tran u\|_2^4} 
& \gtrsim 1 + 4\delta\frac{K-1}{\sqrt{pK}} +
		6\delta^2\frac{K-1}{p} + \bO{\delta^3}
\end{align*}
Using these results, we derive an upper bound for the
expected distortion	of $B$ with $A$ with columns in
$\mathcal E_{\delta,i}$,
\begin{align*}
	\E[D]{\min_{A_i \in \mathcal E_{\delta, i}}
		\left\|S - ABA\tran\right\|_F^2}
	& \le \E[D]{\left\|B\right\|_F^2} - \sum_{i=1}^K
		\E[D]{ \max_{A_i \in \mathcal E_\delta}
			\|DA_i\|_2^4} + C_1\delta^2 + \bO{\delta^3}\\
	& \le K + \frac{K(K-1)}{p} - \sum_{i=1}^K
		1 +  4\delta \frac{K-1}{\sqrt{pK}}
		+ C_2\delta^2 + \bO{\delta^3}\\
	 & \le \frac{K(K-1)}{p}
	 	- 4\delta (K-1)\sqrt{\frac{K}{p}}
	 	+ C'\delta^2 + \bO{\delta^3}\\
\end{align*}
And 
\[
C' = \delta^2\left(4K \E[D]{\|B\|_F^2} - 6\frac{K(K-1)}{p}\right)
\]
This concludes our proof as 
\begin{align*}
	\E[D]{\min_{A_i \in \mathcal E_{\delta, i}}
		\left\|A^{-1}SA - B\right\|_F^2}
		\tag*{\scriptsize(\autoref{lem:A-1_AT})} 
	& = \E[D]{\min_{A_i \in \mathcal E_{\delta, i}}
		\left\|S - ABA\tran\right\|_F^2}+4\delta^2K \E[D]{\|B\|_F^2}\\
		&~~~~~~~~~~~~~~~~~~~~~~~~~~~~~~~~~~~~~~~~~~~~~~
			+ \bO{\delta^3}\\
	& \le \frac{K(K-1)}{p}
	 	- 4\delta (K-1)\sqrt{\frac{K}{p}}
	 	+ C\delta^2 + \bO{\delta^3}~.
	 	\tag*{\scriptsize(\autoref{lem:control_Rf})}
\end{align*}
And 
\[
	C = 8K \E[D]{\|B\|_F^2} - 6\frac{K(K-1)}{p}
\]
	
\end{proof}

%
%
%
\subsection{Controling $\E[z \sZ]{\delta_A(z)}$ }
\label{sub:control_da}

\restatethm{\controlDa}{lem:da}

\begin{proof}
\label{proof:da}
For any random variable $z = (z_1,\dots, z_K) \sim \mathcal Z \in \R^K$ \st{} the $z_i$  are rotational invariant,
then 
\[
	1 = \E[z\sZ]{1} 
	  = \E[z\sZ]{\frac{\|z\|_1}{\|z\|_1}
	  			 \middle|\|z\|_1} 
	  = \sum_{i=1}^K \E[z\sZ]{\frac{|z_i|}{\|z\|_1}
	  						  \middle| \|z\|_1} 
		= K \E[z\sZ]{\frac{|z_1|}{\|z\|_1}
					 \middle| \|z\|_1}
\]
Thus we get:
\begin{equation}
	\label{eq:Ez}
	\E[z\sZ]{\frac{|z_1|}{\|z\|_1}\middle| \|z\|_1}
			= \frac{1}{K}~.
\end{equation}

Let $z = \left( z_1, \dots z_K\right)$ be a vector of
$\R^K$. Then
\[
	\|Az\|_1 = \sum_{i=1}^K
		\left|\sum_{j=1}^K A_{i,j}z_j\right| 
	\le \sum_{i=1}^K \sum_{j=1}^K
		\left| A_{i,j}\right|\left|z_j\right| 
	\le \sum_{j=1}^K
		\left\| A_i\right\|_1\left|z_j\right|
\]

Using \mbox{\autoref{eq:Ez}}, we can compute an upper
bound for $\E[z\sZ]{\frac{\left|Az\right|_1}{\|z\|_1}}$:
\[
	\E[z\sZ]{\frac{\|Az\|1}{\|z\|_1} \middle| \|z\|_1} 
	\le \sum_{i=1}^K \|A_i\|_1
		\E[z\sZ]{\frac{|z_i|}{\|z\|_1} \middle| \|z\|_1} 
	= \frac{\|A\|_{1,1}}{K}
\]
Finally, we can get an upper bound on
$\E[z\sZ]{\|Az\|_1}$: 
\begin{align*}
	\E[z\sZ]{\|Az\|_1}
	& = \E[z\sZ]{ \E[z\sZ]{\|A \frac{z}{\|z\|_1}\|_1
		\middle| \|z\|_1} \|z\|_1}\\
	& \le \E[z\sZ]{ \frac{\|A\|_{1,1}}{K} \|z\|_1}\\
	& \le \frac{\|A\|_{1,1}}{K}\E[z\sZ]{\|z\|_1}
\end{align*}

This permits to control $\E[z\sZ]{\delta_A(z)}$ with
\[
	\E[z\sZ]{\delta_A(z)}
	= \E[z\sZ]{\|Az\|_1 - \|z\|_1}
	\le \frac{\|A\|_{1,1} 
		- \|\I\|_{1,1}}{K} \E[z\sZ]{\|z\|_1}
\]

Then, for $A\in \mathcal E_\delta$, the $\ell_1$-norm of the
columns $A_i$ is 
\[
	\|A_i\|_1 \le  \sqrt{1-\delta^2}
		+ \delta \sqrt{K-1}\nonumber
\]
We can derive an expression of $\frac{\|A\|_{1,1}
- \|\I\|_{1,1}}{K}$ for $\delta \to 0$, 
\begin{eqnarray}
	\frac{\|A\|_{1,1}}{K} - 1
		= \frac{1}{K}\sum_{i=1}^K\|A_i\|_1 - 1 
	& \le & \sqrt{1-\delta^2} + \delta \sqrt{K-1} - 1
		\nonumber\\
	& \le & \delta \sqrt{K-1} - \frac{\delta^2}{2}+
		\underset{\delta\to 0}{\bO{\delta^4}}
	\label{eq:control_dA}
\end{eqnarray}
\end{proof}


%
%
%

\subsection{Evalutation of the gap}
\label{sub:certif}

\begin{propositionA}
	For $A$ with columns chosen greedily in
	$\mathcal E_{\delta,i}$, using results from
	\autoref{lem:control_Rf} and \autoref{lem:da},
	\begin{align*}
	\E[D]{\min_{A\subset \mathcal E_\delta}\left\|A^{-1}SA - B\right\|_F^2\|v\|_2^2
		+ \lambda \delta_A(z)}
	\le &\\ 
	\frac{(K-1)K}{p} \|v\|_2^2
		&+ \delta\sqrt{K-1} \left(\lambda\|z\|_1 -
				\sqrt{\frac{K(K-1)}{p}}\|v\|_2^2 \right)
		+\bO{\delta^2}
	\end{align*}
\end{propositionA}

\begin{proof}
We denote $v = z - z_k$.
For $A$ with columns chosen greedily in
$\mathcal E_{\delta,i}$, using results from
\autoref{lem:control_Rf} and \autoref{lem:da},
\begin{align}
\E[D]{\left\|A^{-1}SA - B\right\|_F^
	  2\|v\|_2^2 + \lambda \delta_A(z)}
\le & \|v\|_2^2 \left(\frac{K-1}{\sqrt{p}}
	\left(\frac{K}{\sqrt{p}} - 4\delta\sqrt{K}\right)
	 	+ \bO{\delta^2}\right)\nonumber\\
	&~~~~+ \lambda \|z\|_1\left(\delta \sqrt{K-1}
		+ \bO{\delta^2}\right)\nonumber\\
\le & \frac{(K-1)K}{p} \|v\|_2^2\nonumber\\
	&~~~~+ \delta\sqrt{K-1} \left(\lambda\|z\|_1 -
			\sqrt{\frac{K(K-1)}{p}}\|v\|_2^2 \right)
	+\bO{\delta^2} \label{eq:gap}
\end{align}
\end{proof}

\restatethm{\certifAcc}{thm:certif}

\begin{proof}
	\label{proof:certif}
	For $A \subset \mathcal E_\delta$ with columns chosen greedily in
	$\mathcal E_{\delta,i}$, using results from
	\autoref{lem:control_Rf} and \autoref{lem:da},
	\begin{equation}
	\begin{aligned}
	\mathcal E_D \left[\min_{A\subset \mathcal E_\delta}\left\|A^{-1}SA - B\right\|_F^2\right.&\|v\|_2^2
		+ \lambda \delta_A(z)\bigg]
	\le &\\ 
	\frac{(K-1)K}{p} \|v\|_2^2
		&+ \delta\sqrt{K-1} \left(\lambda\|z\|_1 -
				\sqrt{\frac{K(K-1)}{p}}\|v\|_2^2 \right)
		+\bO[\delta\to 0]{\delta^2}
	\end{aligned}	
\end{equation}
	Starting from \autoref{prop:cost_update}, and using the results from 
	 as
	\begin{align*}
		\E[D]{F(z_{k+1}) - F(z^*)} \leq&
		\frac{(K-1)K}{p} \|z_k-z^*\|_2^2\\
		&+ \delta\sqrt{K-1} \underbrace{\left(\lambda\left(\|z\|_1+\|z^*\|\right) -
				\sqrt{\frac{K(K-1)}{p}}\|z_k-z^*\|_2^2 \right)}_{\le 0}
		+\bO[\delta\to 0]{\delta^2}
	\end{align*}
\end{proof}

\end{document}